\RequirePackage{amsmath}
\documentclass[letterpaper,11pt]{article}
\usepackage{amsthm}
\usepackage{amsmath,amssymb, mathtools}
\usepackage{framed,caption,color,url}
\usepackage{standalone}
\usepackage{fullpage}
\usepackage{comment}
\usepackage{algorithm2e,algorithmicx,algpseudocode,tikz,wrapfig}
\usepackage{graphicx}               

\usepackage{natbib} 
\usepackage{nicefrac} 
\usepackage{paralist} 
\usepackage{boxedminipage} 
\usepackage[pdfstartview=FitH,colorlinks,linkcolor=blue,filecolor=blue,citecolor=blue,urlcolor=blue,pagebackref=true]{hyperref}
\usepackage{upgreek}

\usepackage{centernot}

\newcommand{\Learn}{\mathsf{Lrn}}
\newcommand{\WRLearn}{\mathsf{WR}}

\newcommand{\Risk}{\mathrm{Risk}}
\newcommand{\Cor}{\mathrm{Cor}}
\newcommand{\CCor}{\mathrm{CCor}}
\newcommand{\hcert}{h_{\mathrm{cert}}}
\newcommand{\hpred}{h_{\mathrm{pred}}}

\newcommand{\Add}{\mathcal{A}dd}
\newcommand{\Rem}{\mathcal{R}em}
\newcommand{\Rep}{\mathcal{R}ep}
\newcommand{\Flp}{\mathcal{F}lip}
\newcommand{\Rob}{\mathrm{Rob}}

\newcommand{\weak}{\mathrm{wk}}

\newcommand{\aSF}{\mathsf{a}}
\newcommand{\gSF}{\mathsf{g}}
\newcommand{\hSF}{\mathsf{h}}

\newcommand{\avr}[2]{\ifthenelse{\equal{#2}{}}{\aSF({#1})}{\ifthenelse{\equal{#2}{0}}{\aSF(\emptyset)}{\aSF({#1}_{\leq #2})}}}

\newcommand{\avrMax}[2]{\ifthenelse{\equal{#2}{}}{\aSF^*({#1})}{\ifthenelse{\equal{#2}{0}}{\aSF^*(\emptyset)}{\aSF^*({#1}_{\leq #2})}}}

\newcommand{\avrApp}[2]{\ifthenelse{\equal{#2}{}}{\tilde{\aSF}({#1})}{\ifthenelse{\equal{#2}{0}}{\tilde{\aSF}(\emptyset)}{\tilde{\aSF}({#1}_{\leq #2})}}}

\newcommand{\avrAppMax}[2]{\ifthenelse{\equal{#2}{}}{\tilde{\aSF}^*({#1})}{\ifthenelse{\equal{#2}{0}}{\tilde{\aSF}^*(\emptyset)}{\tilde{\aSF}^*({#1}_{\leq #2})}}}

\newcommand{\ArgMax}[2]{\ifthenelse{\equal{#2}{}}{\hSF({#1})}{\ifthenelse{\equal{#2}{0}}{\hSF(\emptyset)}{\hSF({#1}_{\leq #2})}}}

\newcommand{\AppArgMax}[2]{\ifthenelse{\equal{#2}{}}{\tilde{\hSF}({#1})}{\ifthenelse{\equal{#2}{0}}{\tilde{\hSF}(\emptyset)}{\tilde{\hSF}({#1}_{\leq #2})}}}

\newcommand{\gain}[2]{\ifthenelse{\equal{#2}{}}{\gSF(#1)}{\gSF(#1_{\leq #2})}}
\newcommand{\gainMax}[2]{\ifthenelse{\equal{#2}{}}{\gSF^*(#1)}{\gSF^*(#1_{\leq #2})}}
\newcommand{\gainApp}[2]{\ifthenelse{\equal{#2}{}}{\tilde{\gSF}(#1)}{\tilde{\gSF}(#1_{\leq #2})}}
\newcommand{\gainAppMax}[2]{\ifthenelse{\equal{#2}{}}{\tilde{\gSF}^*(#1)}{\tilde{\gSF}^*(#1_{\leq #2})}}

\newcommand{\sm}{\setminus}

\newcommand{\RobLearn}{\mathsf{RLrn}}
\newcommand{\RobCertLearn}{\mathsf{CLrn}}

\newcommand{\modelensemble}{h_{\mathsf{ens}}}

\newcommand{\modelknn}{h_{\mathsf{KNN}}}

\newcommand{\loss}{\mathrm{\ell}}

\usepackage{graphicx}

\newcommand{\remove}[1]{}

\newcommand{\se}{\subseteq}

\newcommand{\MRob}{m_{\mathsf{Rob}}}
\newcommand{\MWR}{m_{\mathsf{WR}}}
\newcommand{\MPAC}{m_{\mathsf{Lrn}}}
\newcommand{\MUC}{m^\cH_{\mathsf{UC}}}

\newcommand{\set}[1]{\left\{ #1 \right\}}

\newcommand{\sphere}[1]{{\mathbb S}^{#1}}

\newcommand{\R}{{\mathbb R}}
\newcommand{\N}{{\mathbb N}}

\newcommand{\Adv}{\mathsf{A}}
\newcommand{\AdvC}{\mathcal{A}}

\newcommand{\cD}{{\mathcal D}}

\newcommand{\cH}{{\mathcal H}}
\newcommand{\cI}{{\mathcal I}}

\newcommand{\cN}{{\mathcal N}}

\newcommand{\cS}{{\mathcal S}}

\newcommand{\cX}{{\mathcal X}}
\newcommand{\cY}{{\mathcal Y}}

\usepackage{upgreek}
\usepackage{bm}

\newcommand{\eps}{\varepsilon}

\newcommand{\poly}{\operatorname{poly}}

\newcommand{\Exp}{\operatorname*{\mathbb{E}}}
\newcommand{\Ex}{\Exp}

\newcommand{\Sign}{\operatorname{Sign}}

\newcommand{\argmax}{\operatorname*{argmax}}

\usepackage{bbm}

\newcommand{\one}{\mathbbm{1}}

\newtheorem{theorem}{Theorem}[section]
\theoremstyle{plain}
\newtheorem{claim}[theorem]{Claim}
\newtheorem{proposition}[theorem]{Proposition}
\newtheorem{lemma}[theorem]{Lemma}

\newtheorem{fact}[theorem]{Fact}
\theoremstyle{definition}

\newtheorem{defi/}[theorem]{Definition}

\newenvironment{defi}
  {%
   \pushQED{\qed}\begin{defi/}}
  {\popQED\end{defi/}}

\theoremstyle{definition}
\newtheorem{remark}[theorem]{Remark}

\newcommand{\sdotfill}{\textcolor[rgb]{0.8,0.8,0.8}{\dotfill}} 

\makeatletter
\def\th@protocol{%
    \normalfont 
    \setbeamercolor{block title example}{bg=orange,fg=white}
    \setbeamercolor{block body example}{bg=orange!20,fg=black}
    \def\inserttheoremblockenv{exampleblock}
  }
\makeatother
\theoremstyle{protocol}

\newtheorem{proto}[theorem]{Protocol}
\newtheorem{protoc}[theorem]{Protocol}

\newcommand{\namedref}[2]{#1~\ref{#2}}

\newcommand{\torestate}[3]{%
\expandafter \def \csname BBRESTATE #2 \endcsname{#3}
\theoremstyle{plain}
\newtheorem{BBRESTATETHMNUM#2}[theorem]{#1}
\begin{BBRESTATETHMNUM#2}\label{#2}\csname BBRESTATE #2 \endcsname   \end{BBRESTATETHMNUM#2}
\newtheorem*{BBRESTATETHMNONNUM#2}{\namedref{#1}{#2}}
}

\newcommand{\restate}[1]{\begin{BBRESTATETHMNONNUM#1}[Restated] \csname BBRESTATE #1 \endcsname
\end{BBRESTATETHMNONNUM#1}}

\title{Learning and certification under instance-targeted poisoning\footnote{This is the full version of a paper  appearing in The Conference on Uncertainty in Artificial Intelligence (UAI)~2021.}}
\author{Ji Gao\thanks{University of Virginia,  \href{mailto:jg6yd@virginia.edu}{jg6yd@virginia.edu}. Supported by NSF grant CCF-1910681.} \and  Amin Karbasi\thanks{Yale, \href{mailto:amin.karbasi@yale.edu}{amin.karbasi@yale.edu}. Supported by NSF (IIS-1845032) and ONR (N00014-19-1-2406).} \and Mohammad Mahmoody\thanks{University of Virginia, \href{mailto:mohammad@virginia.edu}{mohammad@virginia.edu}. Supported by NSF grants CCF-1910681 and CNS-1936799.}}

\newcommand{\Mnote}[1]{{\color{red} [{\bf Moh:}  #1]}}

\begin{document}
 
\maketitle

\begin{abstract}
In this paper, we study PAC learnability and certification of predictions under  instance-targeted poisoning attacks, where the adversary who knows the test instance may  change a fraction of the training set  with the goal of fooling the learner at the test instance. Our first contribution is to formalize the problem in various settings and to explicitly model subtle aspects such as the proper or improper nature of the learning,  learner's randomness, and whether (or not) adversary's attack can depend on it. Our main result  shows that when the budget of the adversary scales sublinearly with the sample complexity, (improper) PAC learnability and certification are achievable; in contrast, when the adversary's budget grows linearly with the sample complexity, the adversary can potentially drive up the expected  0-1  loss to one. 

We also study \emph{distribution-specific} PAC learning in the same attack model and show that \emph{proper} learning with certification is possible for learning half spaces under natural distributions. Finally, we empirically study the robustness of $K$ nearest neighbour, logistic regression, multi-layer perceptron, and convolutional neural network  on real data sets against targeted-poisoning attacks. {Our experimental results show that many models, especially  state-of-the-art neural networks,  are  indeed vulnerable to these strong attacks.}  Interestingly, we observe that methods with high standard accuracy might be more vulnerable to instance-targeted poisoning attacks.  
\end{abstract}

\tableofcontents
  
\section{Introduction}
Learning to predict from empirical  examples is a fundamental problem in machine learning. In its classic form, the problem involves a benign setting where the empirical and test examples are sampled from the same distribution $D$. More formally, a learner, denoted by $\Learn$, is given a training set $\cS$, consists of  i.i.d. samples $(x,y)$ from distribution $D$, where $x$ is a data point and $y$ is its label. Then, the learner returns a model/hypothesis $h$ where it will be ultimately tested on a fresh sample from the same distribution $D$.


More recently, the above-mentioned classic setting has been revisited by allowing adversarial manipulations that tamper with the process, while still aiming to make correct predictions. In general, adversarial tampering can take place in both training or testing    of models. Our interest in this work is on a form of training-time attacks,  known as  poisoning or causative attacks~\citep{barreno2006can,papernot2016towards,diakonikolas2019recent,goldblum2020data}. In particular, poisoning adversaries may partially change the training set $\cS$ into another training set $\cS'$ in such a way that the ``quality" of the returned hypothesis $h'$ by the learning algorithm $\Learn$, that is trained on $\cS'$ instead of $\cS$, degrades significantly. Depending on the context, the way we measure the quality of the poisoning attack may change. For instance, the quality of $h'$ may refer to the expected error of $h'$ when test data points are sampled from the distribution $D$. It could also refer to the error on a particular test point $x$, known to the adversary but unknown to the learning algorithm $\Learn$. The latter scenario, which is the main focus of this work, is known as (instance) \emph{targeted poisoning}~\citep{barreno2006can}. In this setting, as the name suggests, an adversary could craft its strategy based on the knowledge of a target instance $x$. Given a training set of $\cS$ of size $m$, we assume that an adversary can change up to $b(m)$ data points, and we refer to $b(m)$ as adversary's ``budget''. Other examples of natural (weaker) attacks may include flipping binary labels, or  adding/removing  data points (see Section~\ref{sec:Defs}). 

Given a poisoning attack, the predictions of a learning algorithm may or may not change. To this end, 
\cite{steinhardt2017certified} initiated the study of \emph{certification} against poisoning attacks, studying  the conditions under which a learning algorithm can certifiably obtain an expected  low risk. To extend these results to the instance-targeted positing scenario,~\cite{rosenfeld2020certified} recently addressed the \emph{instance targeted} (a.k.a., pointwise) certification with the goal of providing certification guarantees about the prediction of \emph{specific}  instances when the adversary can poison the training data.  While the instance-targeted certification has sparked a new line of research~\citep{levine2020deep,chen2020framework,weber2020rab,jia2020intrinsic} with interesting  insights,  the existing works do not address the fundamental question of  when, and under what conditions,  learnability and certification are achievable under the instance-targeted poisoning attack. In this work, we take an initial step along this line and layout the precise conditions for such guarantees.

\paragraph{Problem setup.} Let $\cH$ consists of a hypothesis class of classifiers $h:\cX\rightarrow\cY$ where $\cX$ denotes the instances domain and $\cY$ the labels domain. We would like to study the learnability of $\cH$ under instance-targeted poisoning attacks. But before discussing the problem in that setting, we recall the notion of PAC learning \emph{without} attacks. 

Informally speaking, $\cH$ is ``Probably Approximately Correct'' learnable (PAC learnable for short)   if  there is a   learning algorithm $\Learn$ such that for every distribution $D$ over $ \cX\times\cY$, if $D$ can be learned with $\cH$ (i.e., the so-called realizability assumption holds) then with high probability over sampling any sufficiently large set $\cS \sim D^m$, $\Learn$ maps $\cS$ to a  hypothesis $h \in \cH$ with ``arbitrarily small'' risk under the distribution $D$.  $\Learn$ is called \emph{improper} if it is allowed to output functions outside $\cH$, and it is a \emph{distribution-specific} learner, if it is only required to work when the marginal distribution  $D_{\cX}$  on the instance domain $\cX$ is fixed e.g., to be isotropic Gaussian. (See  Section~\ref{sec:Defs} and Definition~\ref{def:TarPAC} for  formal definitions.)

%
Suppose that  before the example $(x,y) \sim D$ is tested, an adversary who is aware of $(x,y)$ (and hence,  is \emph{targeting} the  instance $x$) can craft a poisoned set $\cS'$ from $\cS$ by \emph{arbitrarily changing} up to $b$ of the training examples in $\cS$. Now, the learning algorithm encounters $\cS'$ as the training set and the hypothesis  it returns is, say, $h'\in\cH$ in the proper learning setting. Now, the predicted label of $x$, i.e., $y'=h'(x)$, may no longer be equal to the correct label $y$.  

\paragraph{Main questions.} In this paper, we would like to study under what conditions on the class complexity $\cH$, budget $b$, and different  (weak/strong) forms of instance-targeted poisoning attacks, one can achieve  (proper/improper) PAC learning. In particular, the learner's goal is to still be correct, with high probability, on \emph{most} test instances, despite the existence of the attack. 
%
A stronger goal than robustness is to also \emph{certify} the predictions $h(x)=y$ with a lower bound $k$ on how much an instance-targeted poisoning adversary needs to change the training set $\cS$ to eventually flip the decision on $x$ into $y' \neq y$.   In this work, we also keep an eye on when robust learners can be enhanced to provide such guarantees, leading to \emph{certifiably robust} learners. 

We should highlight that all the aforementioned methods~\citep{rosenfeld2020certified,levine2020deep,chen2020framework,weber2020rab,jia2020intrinsic} mainly considered practical methods that allow predictions for individual instances under specific conditional assumptions about the model's performance at the decision time that can be only verified empirically, but it is not clear (provably) if such conditions would actually happen during the prediction moment.
In this work, we  avoid such    assumptions and address the question of under what conditions on the \emph{problem's setting}, the learnability is possible provably.

\paragraph{Our contribution.} Our contributions are as follows.

{\textit{Formalism.}} We provide a precise and general formalism for  the notions of certification and PAC learnability  under   instance-targeted attacks. These formalisms are based on a careful treatment of the notions of \emph{risk} and \emph{robustness} defined particularly for learners under instance-targeted poisoning attacks. The definitions carefully consider various attack settings, e.g., based on whether the adversary's perturbation can depend on learner's randomness or not, and also distinguish between various forms of certification (to hold for \emph{all} training sets, or just \emph{most} training sets.)  
    %
    
 {\em Distribution-independent setting.} We then study the problem of robust learning and certification under instance-targeted poisoning attacks in the distribution-independent setting. Here, the learner shall produce ``good'' models for \emph{any} distribution over the examples, as long as the distribution can be learned by at least one hypothesis $h \in \cH$ (i.e., the realizable setting). We separate our studies here based on the subtle distinction between two cases: Adversaries who can base their perturbation also for a \emph{fixed} randomness of the learner (the default attack setting), and those whose perturbation would be retrained using \emph{fresh} randomness (called weak adversaries).
In the first setting, We show that as long as the hypothesis class $\cH$ is (properly or improperly) PAC learnable under the 0-1 loss and the strong adversary's budget is $b=o(m)$, where $m$ is the number of samples in the training set,  then the hypothesis class $\cH$ is always \emph{improperly} PAC learnable under the instance-targeted attack with certification (Theorem~\ref{thm:Learn-Realize-Improper}). This result  is inspired by the  recent work of~\cite{levine2020deep} and comes  with certification.
 We then show that the limitation on $b(m)=o(m)$ is inherent in general, as when $\cH$ is the set of homogeneous hyperplanes, if $b(m)=\Omega(m)$, then robust PAC learning against instance-targeted poisoning is impossible in a strong sense (Theorem~\ref{thm:halfspace-example-realizable}).
 $m$. 
We then show that if the adversary is ``weak'' and is \emph{not} aware of learner's randomness, if the hypothesis class $\cH$ is properly PAC learnable   and the weak adversary's budget is $b=o(m)$,  then $\cH$ is also properly PAC learnable under  instance-targeted attacks (Theorem~\ref{thm:Learn-Realize-Proper}). This result, however, does \emph{not} come with certification guarantees. 
 
 {\emph{Distribution-specific learning.}} We then study robust learning under instance-targeted poisoning when the instance distribution is fixed. We show that when the projection of the marginal distribution $D_{\cX}$ is the uniform distribution over the unit sphere (e.g., $d$-dimensional isotropic Gaussian), the hypothesis class consists of homogeneous half-spaces, and the strong adversary's budget is $b=c /\sqrt{d}$, then proper PAC learnability under instant-targeted attack is possible iff $c=o(m)$ (see Theorems~\ref{thm:halfspace-positive} and~\ref{thm:halfspace-negative}). Note that if we allow $d$ to grow with $m$ to capture the ``high dimension'' setting, then  the mentioned result becomes incomparable to our above-mentioned results for the distribution-independent setting). To prove this result we use tools from measure concentration over the unit sphere  in high dimension.

 {\emph{Experiments.}} We empirically study the robustness of $K$ nearest neighbour, logistic regression, multi-layer perceptron, and convolutional neural network  on real data sets. We observe that methods with high standard accuracy (such as convolutional neural network) are indeed more vulnerable to instance-targeted poisoning attacks. This observation might be explained by the fact that more complex models fit the training data better and thus the adversary can more easily confuse them at a specific test instance. A possible interpretation  is that models that somehow ``memorize'' their data could be more vulnerable to targeted poisoning.
{In addition, we study whether dropout on the inputs and also $L2$-regularization on the output can help the model to defend against instance-targeted poisoning attacks. We observe that adding these regularization   to the learner does not help in defending against such  attacks.}

\subsection{Related work}

The concurrent work of~\cite{blum2021robust} also studies instance-targeted PAC learning. In particular, they formalize and prove positive and negative results about  PAC learnability under  instance-targeted  poisoning attacks, in which the adversary can add an unbounded number of clean-label examples to the training set. In comparison, we formalize the problem for any prediction task and also for certification of results. Our main positive and negative results are, however,  proved for classification tasks and for adversaries who can change  $o(1)$ fraction of the data set.  
Other theoretical works  have also studied instance-targeted poisoning \emph{attacks} (rather than learnability under such attacks) using clean labels~\citep{mahloujifar2017blockwise,mahloujifar2018learning,mahloujifar2019universal,mahloujifar2019can,mahloujifar2019curse,diochnos2019lower,etesami2020computational}.    The work of~\cite{shafahi2018poison} studied such (targeted clean-label) attacks empirically, and showed that  neural nets can be very vulnerable to them. Finally,~\cite{pmlr-v70-koh17a} also studied clean label attacks empirically  for \emph{non-targeted} settings.

More broadly, some classical works in machine learning can also be interpreted as (non-targeted) data poisoning~\citep{valiant1985learning,kearns1993learning,Sloan::Noise:four-types,bshouty2002pac}. In fact, the work of~\cite{bshouty2002pac} studies the same question as in this paper, but for the \emph{non-targeted setting}. However, making learners robust against such attacks can easily lead to \emph{intractable} learning methods that do \emph{not} run in polynomial time. Recently, starting with the seminal results of~\cite{diakonikolas2016robust,lai2016agnostic} and many follow up works   (see the survey~\citep{diakonikolas2019recent}) it was shown that in some natural settings one can go beyond the intractability barriers and obtain polynomial-time methods to resist non-targeted poisoning. In contrast, our work focuses on targeted poisoning.
We shall also comment that, while our focus in this work is on instance-targeted attacks for prediction tasks, it is not  clear how to even define such (targeted) attacks for robust parameter estimation (e.g., learning  Gaussians).

Regarding {certification},~\cite{steinhardt2017certified} were the first who studied certification of the \emph{overall risk} under the poisoning attack. However, the more relevant to our paper is the work by~\cite{rosenfeld2020certified} who introduced the instance-targeted poisoning attack and applied randomized smoothing for certification in this setting. Empirically, they showed how  smoothing can provide robustness against label-flipping adversaries. Subsequently,~\cite{levine2020deep} introduced Deep Partition Aggregation (DPA), a novel technique that uses deterministic bagging  in order to develop robust predictions against general  addition/removal   instance-targeted poisoning. \cite{chen2020framework,weber2020rab,jia2020intrinsic} further developed  {randomized} bagging/sub-sampling  and empirically studied the intrinsic robustness of their methods.
predictions. 


Finally, we note that while our focus is on \emph{training-time-only} attacks, poisoning attacks can be performed in conjunction with test time attacks, leading to  backdoor attacks~\citep{gu2017badnets,ji2017backdoor,chen2018detecting,wang2019neural,turner2019label,diochnos2019lower}.

\section{Definitions} \label{sec:Defs}

\paragraph{Basic definitions and notation.} We let $\N = \set{0,1,\dots}$ denote the set of integers,
$\cX$ the input/instance space,  and $\cY$  the space of labels. By $\cY^\cX$ we denote the set of all functions from $\cX$ to $\cY$.  By $\cH \subset \cY^\cX$ we denote the set of hypotheses.  We use $D$ to denote a distribution over $\cX \times \cY$. By $e \sim D$ we state that $e$ is distributed/sampled according to distribution $D$. For a set $\cS$, the notation $e \sim \cS$ means  that $e$ is uniformly sampled from $\cS$. By $D^m$ we denote a product distribution over $m$ i.i.d. samples from $D$. By $D_\cX$ we denote the projection of $D$ over its first coordinate (i.e., the marginal distribution over $\cX$).
For a function $h \in \cY^\cX$ and an example $e=(x,y)\in \cX \times \cY$, we use $\loss(h,e)$ to denote the loss of predicting $h(x)\in \cY$ while the correct label for $x$ is $y$. Loss will always be non-negative, and when it is in $[0,1]$, we call it bounded. For classification problems, unless stated differently, we use the 0-1 loss, i.e., $\loss(h,e) = \one[h(x)=y]$.  We use $\cS \in (\cX\times \cY)^*$ to denote a training ``set'', even though more formally it is in fact a sequence. 
We use $\Learn$  to denote a learning algorithm that (perhaps randomly) maps a training set $\cS \sim D^m$  of any size $m$ to some $h \in \cY^\cX$. We call a leaner $\Learn$  \emph{proper} (with respect to hypothesis class $\cH$) if it always outputs some $h \in \cH$. $\Learn(\cS)(x)$ denotes the prediction on $x$ by the hypothesis returned by $\Learn(\cS)$. When $\Learn$ is randomized, by $y \sim \Learn(\cS)(x)$ we state that $y$ is the prediction when the randomness of $\Learn$ is chosen uniformly. For a randomized $\Learn$ and the random seed $r$ (of the appropriate length), $\Learn_r$  denotes the deterministic learner with the hardwired randomness $r$.  For a hypothesis $h \in \cH$,  a loss function $\loss$,    and a distribution $D$ over $\cX \times \cY$, the population (a.k.a. true) risk of $h$ over $D$ (with respect to the loss $\loss$) is  defined as  $\Risk(h,D)=\Ex_{e\sim D}[\loss(h,e)]$, and the  empirical risk of $h$ over   $\cS$ is defined as $\Risk(h,\cS)= \Ex_{e\sim \cS}[\loss(h,e)]$. For a hypothesis class $\cH$, we say that the realizability assumption holds for a distribution $D$ if there exists an $h \in \cH$ such that $\Risk(h,D)=0$. To add clarity to the text, We  use a diamond   ``$\Diamond$'' to denote the end of a technical definition. For a hypothesis class $\cH$, we call a data set $\cS \sim D^m$ \emph{$\eps$-representative} if $\forall h \in \cH, |\Risk(h,D)-\Risk(h,\cS)|\leq \eps.$  A hypothesis class has the \emph{uniform convergence} property, if there is a function $m=\MUC(\eps,\delta)$ such that for any distribution $D$, with probability $1-\delta$ over $\cS \sim D^m$, it holds that $\cS$ is $\eps$-representative.

\paragraph{Notation for the  poisoning setting.} For simplicity, we work with deterministic strategies, even though our results could be extended directly to randomized adversarial strategies as well. We use $\Adv$ to denote an adversary who changes the training set $\cS$ into $\cS' = \Adv(\cS)$. This mapping can depend on (the knowledge of) the learning algorithm $\Learn$ or any other information such as a targeted example $e$ as well as the randomness of $\Learn$. 
%
%
By $\AdvC$ we refer to a \emph{set} (or \emph{class}) of adversarial mappings and by $\Adv\in \AdvC$ we denote that the adversary $\Adv$ belongs to this class. (See below for examples of such classes.) Our adversaries always will have a budget    $b \in \N$ that controls how much they can change the training set $\cS$ into $\cS'$ under some (perhaps asymmetric) distance metric. 
To explicitly show the budget, we denote the adversary as $\Adv_b$ and their corresponding classes as $\AdvC_b$. 
Finally, we let  $\AdvC_b(\cS) = \set{\cS' \mid \Adv_b\in\AdvC_b(\cS)}$ be the set of all ``adversarial perturbations'' of $\cS$ when we go over all possible attacks of budget $b$ from the adversary class $\AdvC$.

\paragraph{Adversary classes.} Here we define the main adversary classes that we use in this work. For more noise models see the work of \cite{Sloan::Noise:four-types}.
\begin{itemize} 
    \item{\bf $\Rep_b$ ($b$-replacing).}  The adversary can   replace up to $b$ of the examples in $\cS$ (with arbitrary examples) and then put the whole sequence $\cS'$ in an arbitrary order. More formally, the adversary is limited to (1) $|\cS|=|\cS'|$, and (2) by changing the order of the elements in $    \cS$, one can make the Hamming distance between $\cS',\cS$ at most $b$. This is essentially the targeted version of the ``nasty noise'' model introduced by~\cite{bshouty2002pac}.
    \item {\bf $\Flp_b$ ($b$-label flipping).} The adversary can change the label of up to $b$ examples in $\cS$ and reorder the final set.        
    \item {\bf $\Add_b$ ($b$-adding).} The adversary adds up to $b$ examples to $\cS$ and put them in arbitrary order. Namely, the multi-set $\cS'$  has size at most $|\cS|+b$ and it holds that $\cS \se \cS'$.
    \item {\bf $\Rem_b$ ($b$-removing).} The adversary removes up to $b$ examples from $\cS$ and puts the rest in  an arbitrary order. Namely, as multi-sets $|\cS'| \geq |\cS|-b$ and   $\cS' \subseteq \cS$.
    
    \item {\bf $\Add\Rem_b$ ($b$-adding-or-removing).} The adversary can remove up to $b$ examples from $\cS$, then add up to $b$ arbitrary examples, and then it puts the rest in  an arbitrary order. Namely, as multi-sets $|\cS' \cap \cS| \geq |\cS|-b$ and   $|\cS' \sm \cS| \leq b$.\footnote{$\Rep_b$ attacks are essentially as powerful as $\Add\Rem_b$ attack, with the only limitation that they preserve the training set size.  Our results of Theorems \ref{thm:Learn-Realize-Proper} and \ref{thm:Learn-Realize-Improper} extend to $\Add\Rem_b$ attacks as well, however we focus on $b$-replacing attacks for simplicity of presentation.}
\end{itemize}


We now define the notions of  risk, robustness, certification, and learnability under targeted poisoning attacks for prediction tasks with a focus on classification. 
We emphasize that in the definitions  below, the notions of targeted-poisoning risk and robustness are defined with respect to a  \emph{learner} rather than a hypothesis. The reason is that, very often (and in many natural settings)  when the data set is changed by the adversary, the learner needs to return a   new hypothesis, reflecting the change in the training data,  


\begin{defi}[Instance-targeted poisoning risk] \label{def:TarRisk}
Let $\Learn$ be a  possibly randomized learner,   $\AdvC_b$ be a class of attacks of budget $b$.
For a training set $\cS \in (\cX \times \cY)^m$,  an example $e =(x,y)\in \cX \times \cY$, and   randomness $r$,   the \emph{targeted poisoning loss} (under attacks  $\AdvC_b$) is defined as\footnote{Note that Equation \ref{eq:AdvLoss} is equivalent to 
$\loss_{\AdvC_b}(\cS,r,e)  =  \sup_{\Adv \in \AdvC_b}  \loss(\Learn_r(\Adv(\cS,r,e)), e)$, because we are choosing the attack over $\cS$ after fixing $r,e$.
}
\begin{align}\label{eq:AdvLoss}
\loss_{\AdvC_b}(\cS,r,e)  =  \sup_{\cS' \in \AdvC_b(\cS)}  \loss(\Learn_r(\cS'), e).
\end{align}
For a distribution $D$ over $\cX \times \cY$, the \emph{targeted poisoning risk} is defined  
 as  
\begin{align*}
\Risk_{\AdvC_b}(\cS,r,D) = \Ex_{e \sim D} [\loss_{\AdvC_b}(\cS,r,e)].
    \end{align*}
For a bounded loss function with values in $[0,1]$ (e.g., the 0-1 loss), we   define  the  \emph{correctness} of the learner for the distribution $D$  under targeted poisoning  attacks of $\AdvC_b$ as $$\Cor_{\AdvC_b}(\cS,D)=1-\Risk_{\AdvC_b}(\cS,D).$$
The  above formulation implicitly allows the adversary to depend (and hence ``know'') on the randomness $r$ of the learning algorithm.  We also define  \emph{weak targeted-poisoning} loss and risk   by using   \emph{fresh} learning randomness $r$ unknown to the adversary, when doing the retraining:
\begin{align*}
  \loss^{\weak}_{ \AdvC_b}(\cS,e) =   \sup_{\cS' \in \AdvC_b(\cS)} \Ex_{r}[\loss(\Learn_r(\cS'),e)],~~~~~~~ 
  \Risk^{\weak}_{ \AdvC_b}(\cS,D) = \Ex_{e \sim D} [\loss^{\weak}_{ \AdvC_b}(\cS,e)].
\end{align*}
In particular, having a small weak targeted-poisoning risk under the 0-1 loss means that for most of the points $e \sim D$ the  decisions are correct, and  the prediction on $e$  would not change under any $e$-targeted poisoning  attacks with high probability over a randomized retraining. 
\end{defi}

We  now define   robustness of predictions, which is more natural 
for classification tasks, but we state it more generally. 
\begin{defi}[Robustness  under instance-targeted poisoning] \label{def:TarRob}
Consider the same setting as that of Definition~\ref{def:TarRisk}, and let $\tau >0$ be a threshold to model when the loss is ``large enough''.
For a data set\footnote{Even though, in  natural attack scenarios the set $\cS$ is sampled from $D^m$,    Definitions~\ref{def:TarRisk} and~\ref{def:TarRob} are more general in the sense that   $\cS$  is an arbitrary set. } $\cS$ and learner's randomness $r$, we call  an example $e=(x,y)$ to be \emph{$\tau$-vulnerable} to 
a targeted poisoning (of attacks in $\AdvC_b$), if the $e$-targeted adversarial loss is at least $\tau$, namely, $\loss_{\AdvC_b}(\cS,r,e) \geq \tau$. For the same $(\cS,r,e,\tau)$ we  define the  \emph{targeted poisoning robustness} (under attacks in $\AdvC$) as the smallest budget $b$ such that $e$ is $\tau$-vulnerable to a targeted poisoning, i.e., 
\begin{align*}
    \Rob^\tau_\AdvC(\cS,r,e)  = \inf \set{b \mid \loss_{\AdvC_b}(\cS,r,e) \geq \tau}.
\end{align*}
If no such $b$ exists, we let $\Rob^\tau(\cS,r,e) = \infty$.\footnote{If the adversary's budget allows it to  flip all the labels, in natural settings (e.g., when the hypothesis class contains  the complement functions and the learner is a PAC learner), no robustness   will be infinite for such attacks.} 
When working with the 0-1 loss (e.g., for classification), we will use $\tau=1$ and simply write $\Rob_\AdvC(\cdot)$ instead. Also note that in this case, $\loss(\Learn_r(\cS'),e)\geq 1$ is simply equivalent to   $\Learn_r(\cS')(x)\neq y$. In particular,  if $e=(x,y)$ is an example and $\Learn_r(\cS)$ is already wrong in its prediction of the label for $x$, then the robustness will be $\Rob_\AdvC(\cS,r,e)=0$, as no poisoning will be needed to make the prediction wrong. 
For a distribution $D$  we define the \emph{expected targeted-poisoning robustness}  as
$\Rob^\tau_\AdvC(\cS,r,D) = \Ex_{e \sim D}[\Rob^\tau_\AdvC(\cS,r,e)].$
\end{defi}


We now formalize when a learner provides certifying guarantees for the produced predictions. For simplicity, we state the definition for the case of 0-1 loss, but it can be generalized to other loss functions by employing a threshold parameter $\tau$ as it was done in Definition \ref{def:TarRob}.
\begin{defi}[Certifying predictors and learners] \label{def:TarCer}
A \emph{certifying predictor} (as a generalization of a hypothesis function) is a function $h \colon \cX \to \cY \times \N$, where the second output is interpreted as a claim about the robustness of the prediction. When $h(x)=(y,b)$, we define $\hpred(x)=y$ and $\hcert(x)=b$. If $\hcert(x)=b$, the interpretation  is that the prediction $y$ shall not change when the adversary performs a  $b$-budget   poisoning perturbation (defined by the attack model) over the training set used to train $h$.\footnote{When using a general loss function, $b$ would be interpreted as the attack budget that is needed to increase the loss  over the example $e(x,y)$  (where $y$ is the prediction) to $\tau$.} 
Now, suppose ${\AdvC_b}$ is an adversary class   with budget $b=b(m)$ (where $m$ is the sample complexity) and $\AdvC=\cup_i \AdvC_i$. Also suppose $\Learn$ is a learning algorithm such that $\Learn_r(\cS)$  always  outputs a certifying predictor for any data set $\cS \in (\cX \times \cY)^\star$.
We call $\Learn$ a \emph{certifying learner} (under the attacks in $\AdvC$) for a specific data set $\cS \in (\cX \times \cY)^\star$ and randomness $r$, if the following holds.
For all $x \sim D$, if $ \Learn_r(\cS)(x) = (y,b)$ and if we let $e=(x,y)$,\footnote{Note that $y$ might not be the right label} then $\Rob_{\AdvC}(\cS,r,e) \geq b$. In other words, to change the prediction $y$ on $x$ (regardless of $y$ being a correct prediction or not), any adversary needs a budget at least $b$. We call $\Learn$ a \emph{universal} certifying learner if it is a certifying learning for all data sets $\cS$.
For an adversary class $\AdvC = \cup_{b \in \N} \AdvC_b$, and a   certifying learner $\Learn$ for $(\cS,r)$,  we define the $b$-certified correctness of $\Learn$ over $(\cS,r,D)$ as  the probability of outputting correct predictions while  certifying them with robustness at least $b$. Namely,
\begin{align*}
 \CCor_{\AdvC_b}(\cS,r,D)    = 
 \Pr_{(x,y) \sim D}\left[(y'=y) \land  (b' \geq b)   \text{ where } (y',b')=\Learn_r(\cS)(x)\right].    ~~~~~~\qedhere
\end{align*}  
\end{defi}

\begin{remark}[On a potential weaker requirement for certifying learners]  Definition~\ref{def:TarCer} needs a learner to produce a certifying model that is  \emph{always} correct in its robustness claims about  its own prediction, regardless of whether   the prediction itself is correct or wrong. One can imagine a weaker certification requirement in which the provided certified robustness guarantee is only required to hold when the predicted label itself is correct. However, since a learner usually does not really know whether its prediction is correct with full confidence, known methods for certified robustness already achieve the stronger guarantee of in Definition \ref{def:TarCer}. Also, if one uses that weaker requirement,  \emph{robust} PAC learning and \emph{certified} PAC learning (see Definition \ref{def:TarPAC}) become equivalent, as a learner can simply output $b$ as its certifying guarantee when we know that robust PAC learning against targeted $b$-budget poisoning attacks is possible. 
\end{remark}



The following definition extends the standard PAC   learning framework of \cite{valiant1984theory} by allowing targeted-poisoning attacks and asking the leaner now to have small targeted-poisoning risk. This definition is strictly more general than PAC learning, as the trivial attack that does not change the training set,  Definition~\ref{def:TarPAC} below reduces to  the standard definition of PAC  learning.

\begin{defi}[Learnability under instance-targeted poisoning] \label{def:TarPAC}
Let the function $b \colon \N \to \N$ model adversary's budget as a  function of sample complexity $m$.
A  hypothesis class $\cH$ is  \emph{PAC learnable under targeted poisoning attacks in $\AdvC_b$}, if there is a  proper learning algorithm $\Learn$ such that for every $\eps,\delta \in (0,1)$ there is an integer  $m$ where the following holds. For every   distribution $D$ over $\cX \times \cY$, if the realizability condition holds\footnote{Note that realizability holds while no attack is launched.} (i.e., $\exists h \in \cH, \Risk(h,D)=0$), then with probability $1-\delta$ over the sampling of  $\cS \sim D^m$ and $\Learn$'s randomness $r$, it holds that $\Risk_{\AdvC_b}(\cS,r,D) \leq \eps.$

\begin{itemize}

    \item {\bf Improper learning.} 
We say that  $\cH$  is \emph{improperly} 
PAC learnable under targeted ${\AdvC_b}$-poisoning attacks, if the same conditions as above hold but   using an improper learner that might output functions outside  $\cH$.\footnote{We note, however, that whenever the proper or improper condition is not stated, the default is to be proper.}
 
 \item {\bf Distribution-specific learning.} 
 Suppose $\cD$ is the set of all distributions $D$ over $\cX\times \cY$ such that the marginal distribution of $D$ over its first coordinate (in $\cX$) is a fixed distribution $D_\cX$ (e.g., isotropic Gaussian in dimension $d$).  If all the conditions above (resp. for the improper cases) are only required to hold for distributions $D\in \cD$, then we say that the hypothesis class $\cH$ is PAC learnable (resp. improperly PAC learnable) under instance distribution $D_\cX$ and   targeted ${\AdvC_b}$-poisoning.
\end{itemize}
A hypothesis class is \emph{weakly} (improperly and/or distribution-specific)  PAC learnable   under targeted $\AdvC_b$-poisoning, if with probability $1-\delta$ over the sampling of  $\cS \sim D^m$, it holds that
$\Risk^{\weak}_{ \AdvC_b}(\cS,D) \leq \eps$.
A hypothesis class is \emph{certifiably} (improperly and/or distribution-specific) PAC learnable   under targeted $\AdvC_b$-poisoning, if we modify the $(\eps,\delta)$ learnability condition as follows. With probability $1-\delta$ over $\cS \sim D^m$ and randomness $r$, it holds that (1) $\Learn$ is a certifying learner for $(\cS,r)$, and (2) $\CCor_{\AdvC_b}(\cS,r,D) \geq 1-\eps$. A hypothesis class is \emph{universally} certifiably  PAC learnable, if it is certifiably PAC learnable using a universal certifying learner $\Learn$.
We call the sample complexity of any learner of the forms above \emph{polynomial}, if the sample complexity  $m$ is at most $ \poly(1/\eps,1/\delta)=(1/(\eps\delta))^{O(1)}$. We call the learner \emph{polynomial time}, if it runs in time $\poly(1/\eps,1/\delta)$, which implies the sample complexity is polynomial as well.
\end{defi}

\begin{remark}[Generalization to $(\eps,\delta)$-PAC learning] Suppose $\eps(m),\delta(m)$ are functions of $m$. Then one can generalize Definition \ref{def:TarPAC} to define $(\eps(m),\delta(m))$ PAC learning (under the same settings of Definition \ref{def:TarPAC}) for a given desired $\eps(m),\delta(m)$. Then PAC learnability would simply mean $\eps(m),\delta(m)$ PAC learning for $\eps(m),\delta(m) = o_m(1)$ (i.e., $\eps(m),\delta(m)$ both go to zero, when $m$ goes to infinity). This more fine-grained definition allows one to study \emph{optimal error} bounds in relation to adversary's budget $b(m)$ as well. We leave a more in-depth study of such relations for future work. \end{remark}

\begin{remark}[On defining agnostic learning under instance-targeted poisoning] Definition \ref{def:TarPAC} focuses on the realizable setting. However, one can generalize this  to the agnostic (non-realizable) case by requiring the following to hold with probability $1-\delta$ over $\cS \sim D^m$ and randomness $r$,
$$\Risk_{\AdvC_b}(\cS,r,D) \leq \eps + \inf_{h \in \cH} \Risk(h,r,D).$$ 
Note that in this definition the learner wants to achieve \emph{adversarial} risk that is $\eps$-close to the risk under \emph{no attack}. One might wonder if there is an alternative definition in which the learner aims to ``$\eps$-compete'' with the best \emph{adversarial} risk. However, recall that targeted-poisoning adversarial risk is \emph{not} a property of the hypothesis, and it is rather a property of the learner. This leads to the following arguably unnatural criteria that needs to hold with probability $1-\delta$ over $S \sim D^m$ and $r$. (For clarity the learner is explicitly denoted as super-index for  $\Risk_{\AdvC_b}$.)
$$\Risk^\Learn_{\AdvC_b}(\cS,r,D) \leq \eps + \inf_{L} \Risk^L_{\AdvC_b}(\cS,r,D)$$ 
The reason that the above does not trivially hold is that $\Learn$ needs to satisfy this for \emph{all} distributions $D$ (and most $\cS$) simultaneously, while the learner $L$ in the right hand side can depend on $D$ and $\cS$. 
\end{remark}

\section{Our results}

We now  study the question of learnability  under instance-targeted poisoning. We first discuss our positive and negative results in the context of distribution-independent learning. We then turn to the setting of distribution-dependent setting. At the end, we prove some generic relations between risk and robustness, showing how to derive one from the other.

Due to space limitations, all proofs are moved the full version of this paper \citep{gao2021learning}.
\subsection{Distribution-independent learning}

 We start by showing results on distribution-independent learning. 
%
%
We first show that in the realizable setting, for any hypothesis class $\cH$ that is PAC-learnable, $\cH$ is also PAC learnable under  instance-targeted poisoning attacks that can replace up to $b(m)=o(m)$ (e.g., $b(m) = \sqrt{m}$) number of examples arbitrarily.
%
To state the bound of sample complexity of robust learners, we first define the $\lambda(\cdot)$ function based an adversary's budget $b(m)$.

\begin{defi}[The $\lambda(\cdot)$ function]
Suppose $b(m) = o(m)$. Then for any real number $x$, $\lambda(x)$ returns the minimum $m$ where $m'/b(m') \geq x$ for any $m' > m$. Formally, $$\lambda(x) = \inf_{m \in \mathcal{N}}\left\{\forall m' \geq m, \frac{m'}{b(m')} \geq x\right\}.$$ 
Note that because $b(m) = o(m)$, we have $m/b(m) = \omega_m(1)$, 
so $\lambda(x)$ is well-defined. 
\label{def:lambda}
\end{defi}

\begin{claim}[When $\lambda$ is polynomially bounded] \label{clm:polyLambda}
If $b(m) = O(x^{1-c})$ for any constant $c>0$, then $\lambda(x) = O(m^{1/c})$, which means $\lambda(\cdot)$ is a polynomial function.
For example, when $b(m) =O( \sqrt{m})$, then $\lambda(x) = O(x^2)$.
\end{claim}
\begin{proof}
As $b(m) = O(m^{1-c})$, there exists a number $m_0$ and a constant $q$, that for any $m' \geq m_0$, we have $b(m') \leq q \cdot {(m')}^{1-c}$, which indicates $m'/b(m') \geq q \cdot (m')^c$.
By the definition of $\lambda(x)$, we want to show that for any $m \geq \lambda(x)$, we have $m/b(m) \geq x$. Let $m_1 = (x/q)^{1/c}$, then when $x \geq q \cdot m_0^c$, we have $m_1 \geq m_0$. By Definition~\ref{def:lambda}, $m_1/b(m_1) \geq q \cdot m_1^c = x$. Therefore, $m_1 \in \left\{ \forall m' \geq m, m'/b(m') \geq x \right\} \geq \lambda(x)$. Since $m_1  = O(x^{1/c})$, we have $\lambda(x) = O(x^{1/c})$.
\end{proof}

\begin{theorem}[Proper learning under weak instance-targeted poisoning] \label{thm:Learn-Realize-Proper} Let $\cH$ be the PAC learnable class of hypotheses.
Then, for adversary budget $b (m) = o(m)$, the same class $\cH$ is also  PAC learnable using randomized learners under \emph{weak} $b$-replacing targeted-poisoning attacks. The proper/improper nature of learning remains the same.  Specifically, let $\MPAC(\eps, \delta)$ be the sample complexity of a PAC learner $\Learn$ for $\cH$. Then, there is a learner $\WRLearn$  that PAC learns $\cH$ under weak $b$-replacing attacks with sample complexity at most
\begin{equation*}
 \MWR(\eps, \delta) = \lambda\left(\max\left\{\MPAC^2\left(\eps, \frac{\delta}{2}\right), \frac{4}{\delta^2}\right\}\right).
\end{equation*}
Moreover, if $b(m) \leq O(m^{1-\Omega(1)})$, then whenever $\cH$ is learnable with a polynomial sample complexity and/or a polynomial-time learner $\Learn$, the robust variant $\WRLearn$ will have the same features as well.
\end{theorem} 
 
\begin{proof}[Proof of Theorem~\ref{thm:Learn-Realize-Proper}]

We first clarify that if $b(m) \leq O(m^{1-\Omega(1)})$, and if $\cH$ is learnable with a polynomial sample complexity, then the polynomial sample complexity of the robust variant simply follows from Claim~\ref{clm:polyLambda} and the formula for $\MWR(\eps, \delta)$ as stated in the statement of the theorem. Moreover, the polynomial-time nature of our learner (assuming $\cH$ is polynomial-time learnable) would be straightforward based on its description below.

The idea is to show that even a simple sub-sampling of the right size from the given training set $\cS$, and then training a model over the sub-sample will do what we want. In particular, we will randomly choose $k$ of the elements in $\cS$, call it subset $\cS_k$, and then run any oracle learner for hypothesis class $\cH$. Below, we will first describe how we choose $k$. We will then prove specific properties about the designed learning algorithm, and finally we will analyze its robustness to weak instance-targeted poisoning attacks (who do not know learner's randomness for retraining). We call the new  learner
$\WRLearn$, and denote the oracle that provides learners for $\cH$, simply as $\Learn$. 

Let $k=k(m)=\sqrt{m/b(m)}$. By the definition of $\lambda(x)$, we have that $\forall m \geq \lambda(x)$, $m/b(m) \geq x$. 
For simplicity of notation we might write $k$ and $b$ where both are actually functions of $m$. 

Let $\MPAC(\eps, \delta)$ be the sample complexity of the $\Learn$ which returns a hypothesis with error $\eps$ for at least $1 -\delta$ probability. 
We now show that when the sample complexity $m \geq \MWR(\eps, \delta) = \lambda(\max\{\MPAC^2(\eps, \delta/2), 4/\delta^2\})$ the learner $\WRLearn$ becomes an $(\eps, \delta)$-robust PAC learner. Note that by the definition of $\lambda(\cdot)$, we have
$$\frac{m}{b(m)} \geq \max\left\{\MPAC^2\left(\eps, \frac{\delta}{2}\right), \frac{4}{\delta^2}\right\}.$$ We then have $\sqrt{{m}/{b(m)}} \geq \MPAC(\eps, \delta/2)$ and $\sqrt{{m}/{b(m)}} \geq \frac{2}{\delta}$.

{\em Warm up: PAC learnability without attack.} It holds that $k = \sqrt{m/b} \geq \MPAC(\eps, \delta/2)$. Hence, $\WRLearn(\cS)=\Learn(\cS_k)$ will be a PAC learner which returns a hypothesis of at most $\eps$ with at least $1 - \delta/2$ probability, in the case no attack happens.

{\em Robustness under weak attacks.} Now suppose an adversary can change up to $b$ of the examples through a weak $b$-replacing attack. The probability that the subset $\cS_k$ intersects with any of the $k$ poisoned examples is at most 
$$p(m)=\frac{k \cdot b}{  m} = \sqrt{\frac{m}{b}} \cdot \frac{b }{ m} = \sqrt{\frac{b}{m}} \leq \frac{\delta}{2}.$$ Therefore, with probability at least $1 - p(m)$, none of the poison examples that are introduced by the adversary will land in the subset $\cS_k$.  In this case by a union bound, when learner $\Learn$ is an $(\eps,\delta/2)$ PAC learner, learner $\WRLearn$ will be a $(\eps,\delta/2+p(m))$ PAC learner under weak $b$-replacing instance-targeted poisoning attacks.  As $\delta/2 + p(m) \leq \delta$, $\WRLearn$ with at least $1-\delta$ probability will return a hypothesis that has at most $\eps$ risk under weak $b$-replacing attacks. 
\end{proof}

The above theorem shows that targeted-poisoning-robust proper learning is possible for PAC learnable classes using \emph{private} randomness for the learner if $b (m) = o(m)$. Thus,  it is natural to ask the following question: can we achieve the stronger (default) notion of robustness as in Definition~\ref{def:TarPAC} in which the adversarial perturbation can also depend on the (fixed) randomness $r$ of the learner? Also, can this be a learning with certifications? Our next theorem answers these questions positively, yet that comes at the cost of improper learning. Interestingly, the improper nature of the learner used in Theorem \eqref{thm:Learn-Realize-Improper} could be reminiscent of the same phenomenon  in \emph{test-time} attacks (a.k.a., adversarial example) where,  as it was shown by  \cite{montasser2019vc},  improper learning came to rescue as well.


\begin{theorem}[Improper learning and certification under targeted poisoning] \label{thm:Learn-Realize-Improper}
Let $\cH$ be (perhaps improperly) PAC learnable.
If $b$-replacing attacks have their budget limited to $b(m)=o(m)$, then $\cH$ is  improperly certifiably PAC learnable     under     $b$-replacing targeted poisoning attacks.  Specifically, let $\MPAC(\eps, \delta)$ be the sample complexity of a PAC learner for $\cH$. Then there is a learner $\mathsf{Rob}$ that universally certifiably PAC  learns $\cH$ under $b$-replacing attacks with sample complexity at most
\begin{align*}
    \MRob(\eps, \delta) =  576\lambda\left(\max\left\{\MPAC^2\left(\frac{\eps}{12}, \frac{\eps}{12}\right), \frac{1}{4\eps^2},   \frac{\log\left(\frac{\delta}{2}\right)^2}{\left(\frac{2\sqrt{3}\eps}{3}\right)^4}, \frac{\log_2\left(\frac{2}{\delta}\right)}{576}\right\}\right).
\end{align*}
Moreover, if $b(m) \leq O(m^{1-\Omega(1)})$ and $\cH$ is learnable using a learner with  a    polynomial sample complexity and/or time, the robust variant $\mathsf{Rob}$ will have the same features as well.
\end{theorem}

Before proving Theorem~\ref{thm:Learn-Realize-Proper}, we define the notion of majority ensembles.
\begin{defi}[Majority ensemble]
A \emph{majority ensemble} model $\modelensemble$ is defined over $t$  sub-models
$\{ h_1, \dots, h_t\}$ as follows.  
\begin{equation*}
    \modelensemble(x) = \argmax_{y \in \cY} \sum_{i = 1}^t \one [h_i(x)=y].
\end{equation*}
Where $\one[E]$ is the Boolean indicator function that equals $1$ if $E$ is true. If no strict majority vote exists, then $\modelensemble(x)=\bot$ for some fixed output $\bot$.
\end{defi}

\begin{proof}[Proof of~\ref{thm:Learn-Realize-Improper}]
Similar  to the proof of Theorem~\ref{thm:Learn-Realize-Proper}, if $b(m)=O(m^{1-\Omega(1)})$, the relation between polynomial sample complexity and polynomial time aspects  of the certifying $\mathsf{Rob}$ in relation to the base learner $\Learn$ follows from Claim~\ref{clm:polyLambda}, the polynomial bound $\MRob(\eps, \delta)$, and the description of our learner $\mathsf{Rob}$ below.


Recall that $\Learn$ is a $(\eps', \delta')$ PAC learner and our goal is to show that we can obtain $(\eps,\delta)$-PAC learning under $b$-replacing targeted-poisoning attacks. We will indeed show how to achieve $(O(\eps' + \delta'),O(\eps' + \delta'))$-PAC learning under such attacks.

We first describe a learning method in which the $b$-replacing adversary is \emph{not} allowed to reorder the examples after changing $b$ of the examples in $\cS$. Our robust learner in this case is deterministic. We will then discuss how one can retain the result by handling even when the adversary can reorder the examples. Our robust learner for the latter case is randomized and uses a careful hashing method. This learner is inspired by the randomized method first introduced in \cite{levine2020deep}. In comparison, (1) we need to generalize the hashing method of \cite{levine2020deep} and carefully choose how to hash \emph{repeated} examples in the data set, and (2) we give a proof of generalization based on adversary's budget.

\paragraph{Attacks that do not reorder the examples.} 
We define the operation \textit{partition} with size $k$ as repeatedly collecting first $k$ items in the data set $\cS$ (which is defined as a sequence), that is, when partition data set $\cS=e_1, e_2, \dots, e_m$ with size $k$, the first partition $\cS_1$ will contain examples $e_1, e_2, \dots, e_k$, and the second partition $\cS_2$ will contain examples $e_{k+1}, e_{k+2}, \dots, e_{2k}$. Now, let $t=t(m)=\sqrt{b(m) \cdot m}$. $\RobLearn$ proceeds as follows.
\begin{enumerate}
    \item Partition the data set $\cS$ into $t$ subsets $\cS_1,\dots,\cS_t$ with equal size $m/t$.
    \item For each subset $\cS_i$ where $i \in [t]$, train a sub-model $h_i = \Learn(\cS_i)$.
    \item Returns $\modelensemble$ that is the majority ensemble model of $\{ h_1, \dots, h_t\}$.
\end{enumerate}

If $t=t(m)=\sqrt{m \cdot b(m)}$, $\eps'={\eps}/{12}$, $\delta'={\eps}/{12}$, and 
$p= \max\{\MPAC^2(\eps/12, \eps/12), \\ 144/\eps^2,  {-\log(\delta)}/{\left(2(\eps/12)^2\right)}\},$
we show that
$\lambda(p)$ becomes an upper bound on the sample complexity $m$ of a robust PAC learner under $b$-replacing attacks. By the definition of the function $\lambda(\cdot)$, we have  $m/b(m) \geq p$.
Therefore, we have $\sqrt{{m}/{b(m)}} \geq \MPAC(\eps/12, \eps/12)$, $\sqrt{{m}/{b(m)}} \geq 12/\eps$, and $\sqrt{{m}/{b(m)}} \geq {-\log(\delta)}/{\left(2(\eps/12)^2\right)}$. For simplicity of notation we might write $t$ and $b$ directly where both are actually functions of $m$. 

We start by showing the learner $\RobLearn$ has the following two properties:
\begin{itemize}
    \item PAC learnability of each sub-model without attack: Each set $\cS_i$ has $m/t$ examples. Therefore, eventually all the partition sets $\cS_i, i \in [t]$ will have enough examples for PAC learning. Specifically, $m/t = \sqrt{m/b} \geq \MPAC(\eps/12, \eps/12)$.
    \item Not many sub-models are under attack: An adversary who can replace $b$ examples in these $t$ sets, is indeed affecting only $t/b$ fraction of the subsets, and $t = \sqrt{b \cdot m}$, ${b}/{t} = \sqrt{{b}/{m}} \leq {\eps}/{12}$.

\end{itemize}

The above arguments show that for each sub-model $h_i$, we can guarantee $(\eps',\delta')$-PAC learning using the number of samples $\MPAC(\eps/12, \eps/12)$ that falls into the corresponding $\cS_i$. Then, we want to argue that the ensemble $\modelensemble$, which is the majority applied to $h_1,\dots,h_t$, is indeed $(O(\eps'+\delta'),O(\delta'+\eps'))$-PAC learning even under $b$-budget changing adversaries (who do not reorder the new set $\cS'$).

We will first argue about why the obtained ensemble model \emph{without attack} has small risk, and once we do it, we argue why it has small risk even under $b$-replacing attacks who do not reorder the output examples.

We start by showing that with high probability, most sub-models have small risk. One might be tempted to use the union bound and conclude that with probability $1-t\cdot \delta'$ all of $h_1,\dots,h_t$ have risk at most $\eps'$, before arguing about the low risk of their majority. But this is a lose confidence bound as $t\cdot \delta'$ can grow to be larger than one. Hence, we need a more careful analysis. In particular, we use concentration bounds to conclude that with high probability \emph{most} of the sub-models have risk at most $\eps'$. Namely, using the Hoeffding inequality, we can conclude that with probability at least $1 - e^{-2t\cdot \delta'^2}$, it holds that the fraction of $h_1,\dots,h_t$ with risk at most $\eps $ is at most $2\delta'$. When $m \geq \MRob(\eps, \delta)$, we have $t=\sqrt{m \cdot b(m)} \geq \sqrt{m/b} \geq {-\log(\delta)}/{\left(2(\eps/12)^2\right)} = {-\log(\delta)}/{\left(2\delta'^2\right)}$. As $1 - e^{-2t\cdot \delta'^2} \geq 1 - e^{-2 \cdot ({-\log(\delta)}/{2\delta'^2}) \cdot \delta'^2 } = 1 - \delta.$ In that case, we can argue about the robustness of the majority ensemble as follows.

Recall that at this stage we are assuming $1-2\delta'$ fraction of the models $h_1,\dots,h_t$ have risk at most $\eps$. We claim that if we let $\eps = 3(2\delta' + \eps')$, then with probability at least $1-\eps'$ over $e=(x,y) \sim D$, it holds that at least $2t/3$ of the sub-models $h_1,\dots,h_t$ give the right answer $y$ on instance $x$. Otherwise we can derive a contradiction as follows. Suppose more than $\eps$ fraction of the examples $e=(x,y) \sim D$  have at least $t/3$ wrong answers among $h_1,\dots,h_t$, i.e., $\Pr_{(x, y) \sim D}\left[\sum_{i=1}^t \one[h_i(x) \neq y] \geq {t}/{3}\right] > \eps$. Then, when we pick both  $i \sim [t]$, and $e=(x,y) \sim D$ at random and get $h_i(x)$ as answer, we get a wrong answer with probability more than $\eps/3$. On the other hand, this probability cannot be too large, because at most $2\delta'$ fraction of $i \sim [t]$ give a model $h_i$ with risk more than $\eps'$, and the rest have risk at most $\eps'$, and hence we should have $\eps/3 <2\delta' + \eps'$, which contradicts $\eps = 3(2\delta'+\eps')$.

Now, we argue that essentially the same bounds above hold even if an adversary goes back and changes $b$ of the examples among the all $m$ examples based on knowing a test example. The only place in the proof that we need to modify  is where we obtained $\eps/3 \leq 2\delta' + \eps'$, while now we shall allow the adversary to corrupt $b$ of the $t$ sub-models by planting wrong examples into their pool $\cS_i$. This can only corrupt $b/t$ fraction of the $t$ models, leading to the bound $\eps=3(2\delta' + 
b/t + \eps')$.

As a summary, with $\eps'={\eps}/{12}$ and $\delta'={\eps}/{12}$, when $\MRob(\eps, \delta) = \lambda(p)$ and $t = \sqrt{b \cdot m}$, the majority learner is an $(\eps, \delta)$-PAC learner to $b$-replacing attacks that do not reorder the examples, as with probability at least $1 - e^{-2t\delta'^2} \geq 1 - \delta$, the robust risk of the learner is at most 
\begin{align*}
    3\left(2\delta' + \frac{b}{t} + \eps'\right)   = 3\left(2 \cdot \frac{\eps}{12} + \sqrt{\frac{b}{m}} + \frac{\eps}{12}\right)   \leq 3\left(2 \cdot \frac{\eps}{12} + \frac{\eps}{12} + \frac{\eps}{12}\right) = \eps.
\end{align*}

\paragraph{Adding certification.} Finally, we define a certifying model $\hcert$ that returns certifications larger than $b$ with high probability. Let $$\hcert(x) = \sum_{i=1}^t \one\left\{h_i(x) = y'\right\} - \frac{t}{2}$$ where $y' = \modelensemble(x)$ and $h_1, \dots, h_t$ are sub-models in $\modelensemble$. As the sub-models $h_1, \dots, h_t$ are trained with separate data sets, for any $b' < \hcert(x)$, the prediction of $\modelensemble$ remains the same, indicates that $\hcert$ always gives correct certification. Now, from the previous analysis, we have $$\Pr_{\cS}\left[\CCor_{\Rep_b}(\cS, D) \geq 1 - \eps\right] \geq 1- \delta.$$
Therefore, $\cH$ is certifiably  PAC learnable under $\Rep_b$ attacks with the aforementioned upper bound on its sample complexity.

 
\paragraph{Attacks that might reorder the examples.} The above learner was indeed deterministic, but it leveraged on the fact that the adversary will not reorder the examples, hence most sub-models are robust to adversarial perturbations. For the full-fledged $b$-replacing adversaries, we will use randomness $r$ that (informally speaking) defines a hash function from $\cX \times \cY$ to $[t]$. The hash function can either be a random oracle, or an $m$-wise independent function (for sake of a polynomial-time learner). We then partition the training set $\cS$ into $t$ subsets by using the hash function that looks at \emph{individual} examples to determine where they land among the $t$ subsets $\cS_1,\dots,\cS_t$.

Because we did not make any assumptions about distribution $D$, the training set $\cS$ could have multiple instances of the same input if $D$ is concentrated on some examples. If we simply pick a hash function $h$ to map $\cX \times \cY$ to $[t]$, it might make the subsets unbalanced and thus lose the i.i.d. property of the distributions generating subsets $\cS_i$.

We then slightly revise the rule to evenly distributed these examples as follows.
For an example $e_i = (x_i, y_i)$ in the training set $\cS$, let $O_i$ be the number of occurrence of the same example $(x_i, y_i)$ in $\cS$ ($0$ if it's the first occurrence). We then use a hash function family $h_K: \cX \times \cY \times [m] \rightarrow [t]$, where $K$ is a  key generated by $r$. The $j$-th occurrence of $e_i$ is then mapped into the partition $t_i$ where $t_i = h_K(e_i, j)$.

Following our assumption of the hash function being independently random on all elements in $\cS$, each partition $\cS_i$ is now an i.i.d. sample of the same distribution. It is because each example in $\cS_i$ is independently and identically sampled from $\cS$, which is an i.i.d. sample of $D$. Therefore, with enough number of examples in $\cS_i$, by the PAC learnablity of $\cH$, each sub-model $h_i$ will be a PAC learner. However, for a pair of $(\eps, \delta)$, it is not guaranteed that $\cS_i$ has enough number of examples for $(\eps, \delta)$-PAC learning, because we are using a probabilistic hashing. If some of the sub-models do not have enough examples in their pool $\cS_i$, it is then hard to show the majority ensemble model is a good model with error less than $\eps$. To handle this problem, we only train sub-models on the partitions with enough number of examples. 

We pick $t = 4\sqrt{b(m) \cdot m}$ be the number of subsets. $\RobLearn$ proceeds as follows.

\begin{enumerate}
    \item For the $j$-th occurrence of the example $e_i \in \cS$, add it into partition $\cS_{t_i}$ where $t_i = h_K(e_i, j)$.
    \item For each subset $\cS_i$ that $|\cS_i| \geq {m}/{6t}$ where $i \in [t]$, train a sub-model $ \Learn(\cS_i)$. 
    \item Denote all the sub-models trained in Step 2 as $h_1, h_2, \dots, h_{t'}$.
    \item Return $\modelensemble$, the majority ensemble model of $\{ h_1, \dots, h_{t'}\}$.
\end{enumerate}

Here, the majority ensemble model will have $t'$ (instead of $t$) sub-models, and $t' \leq t$. We now show that when $p'=\max\left\{\MPAC^2\left({\eps}/{12}, {\eps}/{12}\right), {1}/{4\eps^2}, \log(\delta/2)^2/((2\sqrt{3}\eps/3)^4), \log_2(2/\delta)/576\right\}$ with the sample complexity bounded by  $m \geq \MRob(\eps, \delta) = 576 \lambda(p')$, $\RobLearn$ is robust to $b$-replacing attacks that can reorder the examples. 
    

First, we prove that the majority of the partitions $\cS_i$ will have enough samples, specifically, at least $t' \geq {t}/{4}$ sub-models will have ${m}/{6t}$ examples with high probability.

To analyze the probability of $t' \geq {t}/{4}$, we first consider a simple bucket and ball setting. Consider there are $2t$ examples (balls) and we partition them into $t$ subsets (buckets). Then the probability that at least $t/2$ buckets are not empty is at least  $$1 - \binom{t}{t/2}\left(\frac{1}{2}\right)^{2t} = 1 - \binom{t}{t/2}\left(\frac{1}{2^{t}}\right) \cdot \left(\frac{1}{2^t}\right) \geq 1 - \frac{1}{2^t}.$$ It is because if there are $t/2$ empty buckets, then all $2t$ balls should be in the other $t/2$ buckets. The probability is then calculated by taking a union bound over all $\binom{t}{t/2}$ choices of $t/2$ empty buckets in $t$ buckets.

Now, we have $m$ examples in total. We then consider $m$ examples as $m/2t$ rounds of $2t$ examples. Then for each round, at least $t/2$ subsets have at least one example with probability at least $1 - {1}/{2^t}$. Clearly, applying the union bound over all the rounds of examples gives the result that with probability $1 - {m}/{(2t \cdot 2^t)}$, every round makes at least $t/2$ buckets non-empty. Then, by a simple counting argument, at the end at least $t/4$ buckets will have at least $m/6t$ examples. (Otherwise, the total number of examples would be fewer than $(t/2)(m/3t)$.)




We now prove some properties  for $\RobLearn$. Let $\eps' = \delta' = \frac{\eps}{12}$, when $m = \lambda(p')$. Then, we have
\begin{itemize}
    \item Not many sub-models are under attack: An adversary who can corrupt $b$ of these $t/4$ sets, is indeed corrupting only $4b/t$ fraction of them. We then have $4b/t = \sqrt{b/m} \leq \delta'$.
    \item PAC learnability of each sub-model without attack: The sub-model that has $m/6t$ examples have enough examples for PAC learning. $m/6t = \sqrt{m/b}/24 \geq \MPAC^2(\eps/12, \eps/12)$.
    \item Enough examples: With probability $ 1 - {m}/{(2t \cdot 2^t)}$, at least $t/4$ subsets have at least $m/6t$ examples. We have ${m}/{(2t \cdot 2^t)} < {1}/{2^{(\log_2(2/\delta))}} = \delta/2$
    \item  Most sub-models have low risk: By Hoeffding's inequality, with probability at least $1 - e^{-2t\cdot \delta'^2}$, it holds that the fraction of $h_1,\dots,h_{t'}$ with risk at most $\eps'$ is at most $2\delta'$. When $m \geq \MRob(\eps, \delta)$, we have $t' \geq t/4 \geq \sqrt{m/b}/4 \geq {-\log(\delta)}/{\left(8\delta'^2\right)}$
\end{itemize}

In summary, we show that with probability at least $1 - \delta/2$ , we have at least $t/4 = \sqrt{m \cdot b(m)}$ subsets, each subset has at least $\MPAC(\eps/12, \eps/12)$ examples, and we train an majority ensemble model on it. We then follow the same analysis from the case that the attacks can not reorder the examples. Therefore, with probability at least $1-\delta/2$, $\RobLearn$ is a $(\eps, \delta/2)$-PAC learner under $b$-replacing attacks. By the union bound, $\RobLearn$ is a $(\eps, \delta/2)$-PAC learner under $b$-replacing attacks.

As a summary, ensemble learner $\RobLearn$ achieves a bound similar to the sample complexity bound of the non-reordering attacks.  When $\MRob(\eps, \delta) =  576\lambda\left(p'\right)$, the majority learner is robust to $b$-replacing attacks that can also reorder the examples.

Finally, when $m \geq 576\lambda\left(p'\right)$, certifying model $\hcert(\modelensemble, x) = \sum_{i=1}^{t'} \one\{h_i(x) = y'\} - t'/2$ gets $$\Pr_{\cS}\left[\CCor_{\Rep_b}(\cS, D) \geq 1 - \eps\right] \geq 1- \delta$$ over data set $\cS$. Therefore, $\cH$ is certifiably PAC learnable under $\Rep_b$ attack.
\end{proof}
 
\paragraph{Extension to $\Add\Rem_b$ attacks.} The proofs of Theorems \ref{thm:Learn-Realize-Proper} and \ref{thm:Learn-Realize-Improper}     extend to $\Add\Rem_b$ attacks as well when $b=o(m)$. This is because, at a high level, all we care about is that adversarial ``changes'' (whether they are addition or removal of examples) either do not hit the sub-sampled dataset (in Theorem \ref{thm:Learn-Realize-Proper}) or hit few of the sub-samples (in Theorem \ref{thm:Learn-Realize-Improper}).

We then show that limiting adversary's budget to $b(m)=o(m)$ is essentially necessary for obtaining positive results in the distribution-independent PAC learning setting, as some  hypothesis classes with finite-VC dimension are not learnable under targeted poisoning attacks  when $b(m)=\Omega(m)$ in a very strong sense: any PAC learner (without attack) would end up having essentially a risk arbitrary close to $1$ under attack for any $b(m)=\Omega(m)$ budget given to a $b$-replacing adversary.

We use homogeneous halfspace classifiers, defined in Definition~\ref{def:halfspace} below, as an example of hypothesis classes with finite VC dimension. Then in Theorem~\ref{thm:halfspace-example-realizable}, we show that the hypothesis class of halfspaces are not distribution-independently robust learnable against $\Omega(m)$-label flipping instance-targeted attacks.
\begin{defi}[Homogeneous halfspace classifiers]     \label{def:halfspace}
    A (homogeneous) halfspace classifier $h_\omega: \R^d \rightarrow \{0, 1\}$ is defined as  $h_\omega(x) = \Sign(\omega \cdot x)$, where $\omega$ is a $d$-dimensional vector. We then call $\cH_{\mathsf{half}}$ the class of halfspace classifiers $\cH_{\mathsf{half}} = \{ h_\omega(x) \colon \omega \in \R^d \}$. 
For simplicity, we may use $\omega$ to refer to both the model parameter and the classifier. 
\end{defi}

\begin{theorem}[Limits of distribution-independent learnability of halfspaces] \label{thm:halfspace-example-realizable}
Consider the halfspaces hypothesis set $\cH=\cH_{\mathsf{half}}$ and we aim to learn any distribution over the unit sphere using $\cH$. Let the adversary class be $b$-replacing with $b(m)=\beta \cdot m$ for any (even very small) constant $\beta$. For any (even improper) learner $\Learn$ one of the following two conditions holds.
Either $\Learn$ is \emph{not} a PAC learner for the hypothesis class of half spaces (even without attacks) or there exists a distribution $D$ such that    $\Risk_{\Flp_b}( \cS, D) \geq 1 -  \sqrt{\sigma}$ with probability $1 - \sqrt{\sigma}$ over the selection of $\cS$ of sufficiently large $m \geq \MPAC(\beta\cdot \sigma/6 , \sigma/2)$, where $\MPAC$ is the sample complexity of PAC learner $\Learn$.
\end{theorem}



\begin{proof} [Proof of Theorem~\ref{thm:halfspace-example-realizable}]
To prove the theorem, we select a distribution $D$ and an $\Omega(m)$-label flipping adversary, that for any PAC learner $\Learn$, the targeted poisoning risk is high. We first prove the theorem for the ERM rule, and then we discuss how it extends to any PAC learner.

Our scenario is in dimension $d=3$ with dimensions $X,Y,Z$. Consider the following distribution $D$: For $e = (\alpha, c) \sim D$ where $\alpha$ is a point in the 3-dimensional space and $c$ is a label in $\{+1, -1\}$, with probability $1/2$ we sample $\alpha$ uniformly from the unit circle with $z=1$ (namely $x^2+y^2=1, z=1$) and we let label $c = +1$ of the sampled point $\alpha$. In addition, with probability $1/2$ we sample $\alpha$ uniformly from the unit circle $x^2+y^2=1, z=-1$ and let label $c = -1$. This distribution is realizable over the halfspaces hypothesis set, as halfspace $\omega = (0, 0, 1)$ has $0$ risk on $D$. In the following analysis, we call an arc of one of the circles as an interval $\cI$. We then define the measure of the interval $\cI$ as the probability that a random example $\beta \gets D$ that falls into the interval. Clearly in our setting, an interval $\cI$ can be uniquely determined by fixing its measure $\beta$ and its center point $\alpha'$. This scenario is shown in Figure~\ref{fig:counter_example}.

\begin{figure}
    \centering
    \includegraphics[width=0.3\textwidth]{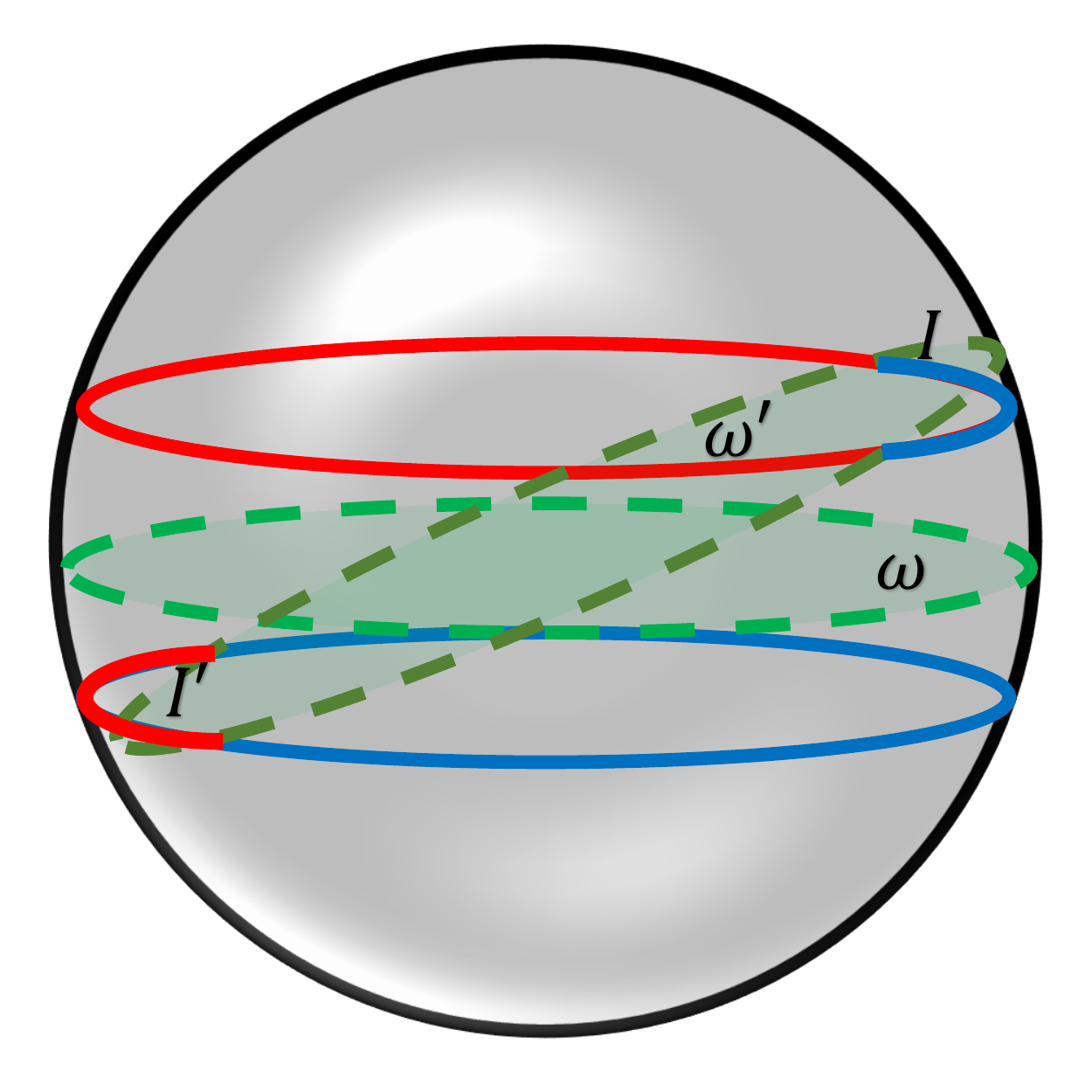}
    \caption{Example for proving Theorem~\ref{thm:halfspace-example-realizable}. The red circle has label $1$, and the blue circle has label $-1$. $\omega$ is the ground-truth halfspace with $0$ risk, and $\omega'$ is the halfspace that has $0$ risk after adversary make replacements.}
    \label{fig:counter_example}
\end{figure}

Now, assume the adversarial perturbation $\cS' = \Flp_b(\cS)$ (that depends on $e = (\alpha, c)$) wants to fool the learner on the point $\alpha=(x,y,z)$. We now define the adversary $\Adv_b(\cS)$, that with a data set $\cS \sim D^m$ and target point $\alpha$, the adversary operates as the following.

\begin{itemize}
    \item Pick an interval $\cI$ of constant measure $\beta/3$ which is centered at $\alpha \in \cI$  \emph{in the same circle where $\alpha$ belongs}. 
    \item To make the attack realizable, pick another corresponding interval, where $\cI' = \set{\alpha' \mid -\alpha' \in \cI}$.
    \item For all $(\alpha_i, c_i) \in \cS$, flip the label if $\alpha_i \in (\cI \cup \cI')$. Return the new set as $\cS'$.
\end{itemize}

In total $\cI$ and $\cI'$ has probability measure ${2\beta}/{3}$. Each example in $\cS$ has probability ${2\beta}/{3}$ to fall into $\cI \cup \cI' $. Then by the Hoeffding's inequality, \begin{equation*}
    \Pr\left[|\cS \cap \cS' | \leq  (1 - \beta) \cdot m \right] \geq 1 - e^{-\frac{m}{18}}.
\end{equation*}

That is, with high probability $\Adv_b(\cS)$ will modify less than $b(m) = \beta \cdot m$ examples. We then analyze how this adversary fools the learners.


\paragraph{ERM learner.} 

We start from the case that the learner is the ERM learner. As $\cI$ and $\cI'$ are symmetric to the origin $(0, 0, 0)$, there exists a halfspace $\omega' \in \cH_{\mathsf{half}}$ that passes all the endpoints of arcs $\cI$ and $\cI'$, which then has $0$ empirical risk on $\cS'$. With probability at least $1 - 2(1 - \beta/6)^m \approx 1$, $\cS$ contains two examples from $\cI$ that positioned  at either side around $\alpha$, that $\omega$ (and all other hypothesis that correctly predicts $\alpha$) will have non-zero risk on $\cS'$. Therefore, ERM will return a hypothesis that incorrectly predicts $\alpha$.


\paragraph{Extension to any proper PAC learner.} 

We now prove that the same adversary can fool any proper PAC learner with sufficiently large $m$.
Let $D'$ be the ``poisoned'' distribution, that is, for $(\alpha_1, c_1) \in \textsf{supp}(D')$ and $(\alpha_1, c_2) \in \textsf{supp}(D)$. 
$c_1 = \begin{cases}
-c_2  &\alpha' \in \cI \cup \cI' \\
c_2 &\textsf{Otherwise}
\end{cases}$. Then for $\cS' = \Adv_b(\cS)$, when $\cS \sim D^m$, $\cS' \sim (D')^m$.

Now, let $ \MPAC(\eps_1, \delta_1)$ be the sample complexity of $\cH_{\textsf{half}}$ on $D'$. When $m \geq \MPAC(\eps_1, \delta_1)$, on the distribution $D'$, $\Learn(\cS')$ holds $\Risk(\Learn(\cS'), D') \leq \eps_1$ with probability at least $1 - \delta_1$. 

Let $\eps_1 = {\beta}/{4}$.
Because hypothesis set $\cH$ are halfspaces, the prediction region (the subset of all the examples predicted for a specific label) is also a connected interval. Therefore, if $\Learn(\cS')$ incorrectly predicts $\alpha$ on $\cS'$ (which is, correctly predicts $\alpha$ on the original data set $\cS$),  as $\alpha$ is at the center of $\cI$, at least half of $\cI$ (and $\cI'$ because of symmetry) is incorrectly predicted, i.e., $\Risk(\Learn(\cS'), D') \geq {\beta}/{3}$. This contradicts $\Risk(\Learn(\cS'), D') \leq \eps_1 = {\beta}/{4}$. 
Therefore, for the selected values of $\eps_1$ and  $\delta_1$, with a sufficiently large sample complexity $m \geq \MPAC({\beta}/{4}, \delta_1)$, the probability of $\alpha$ being misclassified becomes at least $1 - \delta_1$, which indicates the adversary succeeds with probability at least $1 - \delta_1$. By averaging, with probability at least $1 - \sqrt{\delta_1}$, we have $\Risk_{\Flp_b}(\cS, D) \geq 1 - \sqrt{\delta_1}$.

\paragraph{Extension to any improper PAC learner.} 

Previous method cannot be directly applied to improper PAC learners as we no longer have at least half of $\cI$ is incorrectly predicted if $\alpha$ is incorrectly predicted. We now slightly revise $\Adv_b(\cS)$ to fool improper PAC learners as well.

To fool an arbitrary improper PAC learner, the adversary will \emph{randomize} the interval $\cI$. 
The revised adversary $\Adv'_b(\cS, \alpha)$ works as the following. 
\begin{itemize}
    \item Compute the interval $\cI_0$ which is centered at $\alpha$ with measure $\beta/3$.
    \item Uniformly pick a random point $\alpha_r$ from $\cI_0$.
    \item Pick the intervals $\cI$ symmetrically around $\alpha_r$ with measure $\beta/3$, and let $\cI' = \{\beta | -\beta\in \cI\}$.
\end{itemize}

We have $\cS' = \Adv'_b(\cS)$ where $\cS \sim D^m$. Now, let $D'_{\cI}$ be the data distribution where the labels of the examples in $\cI$ and $\cI'$ are flipped, we have $\cS' \gets {D'_{\cI}}^m$ as one can view the poisoned data set $\cS'$ as an i.i.d. sample from the poisoned distribution $D'_{\cI}$, which is conditioned on $\cI$ and $\cI'$. $\cI$ and $\cI'$, on the other hand, is conditioned on the poisoning target $\alpha$.


Now, consider a different process that generates the variables in a different order, that the adversary first uniformly picks a interval $\cI$ among all the interval with measure $\beta/3$ (and its counterpart $\cI'$), and then uniformly samples an example $\alpha$ inside $\cI$ and $\cI'$. Because the sampling is uniform, the probability of picking a specific combination of $\cI$, $\cI'$ and $\alpha$ in the second process is equivalent to the probability of picking this combination following the original process, i.e., pick a random $\alpha$, and then pick $\cI$ conditioned on $\alpha$. Because this equivalence, if $\alpha$ is picked \textit{after} the learner returns a model learned from the data set $\cS'$ (since it is sampled from $D'_{\cI'}$), the probability of whether $\Learn(\cS')$ incorrectly predicts $\alpha$ remains the same.

We now prove that when $m$ is sufficiently large, attacks succeed with high probability on improper PAC learners. Let $ \MPAC(\eps_1, \delta_1)$ be the sample complexity of $\cH_{\textsf{half}}$ on $D'_{\cI}$. When $m \geq \MPAC(\eps_1, \delta_1)$, on the distribution $D'$, $\Learn(\cS')$ holds $\Risk(\Learn(\cS'), D'_{\cI}) \leq \eps_1$ with probability at least $1 - \delta_1$. Since we can equivalently assume $\alpha$ is sampled after $\Learn(\cS')$ is done, the probability of $\Learn(\cS')$ correctly predicts $\alpha$ on $D'_{\cI}$ (which is, incorrectly predicts $\alpha$ on $D$) is at least $1 - \epsilon_1/ (\beta / 3)$. Let $\epsilon_1 = \sigma \cdot \beta/6$ and $\delta_1 = \sigma/2$.


Therefore, for the selected values of $\eps_1$ and  $\delta_1$, with $m \geq \MPAC(\eps_1, \delta_1)$, the probability of $\alpha$ being misclassified becomes at least $1 - \epsilon_1/(\beta / 3) - \delta_1 = 1 - \sigma/2 - \sigma/2 = 1 - \sigma$. By averaging, with probability at least $1 - \sqrt{\sigma}$, we have $\Risk_{\Flp_b}(\cS, D) \geq 1 - \sqrt{\sigma}$.
\end{proof}

\begin{remark}[On $(\eps,\delta)$-PAC learning with $\eps=\Omega(1)$] Theorem~\ref{thm:halfspace-example-realizable} shows that if adversary's budget scales linearly with the sample complexity $m$, then one cannot get $(\eps,\delta)$ PAC learners that are robust against instance-targeted poisoning attacks and that $\eps,\delta=o_m(1)$. However, one can also ask what is the minimum achievable error $\eps(m)$, perhaps as a function of adversary's budget $b(m)$, even when $b(m)=\Omega(m)$. For example, what would be the optimal learning error, if adversary corrupts 1\% of the examples. The same proof of Theorem~\ref{thm:halfspace-example-realizable} shows that in this case, any learner that is robust to instance-targeted $\Rep_b$ attacks would need to have $\eps(m) = \Omega(b(m)/m)$. The reason is that if $\eps(m) = o(b(m)/m)$, then one can still choose $\sigma_m=o_m(1)$, while $\eps(m)=(b(m)/m)\cdot \sigma(m)/6 ,\delta(m)= \sigma(m)/2$ are both $o_m(1)$ as well.

\end{remark}

Note that it was already proved by \cite{bshouty2002pac} that, if the adversary can corrupt $b=\Omega(m)$ of the examples, even with \emph{non-targeted} adversary, robust PAC learning is  impossible. However, in that case, there is a learning algorithm with error $O(b/m)$. So if, e.g., $b=m/1000$, then non-targeted learning is possible for practical purposes. On the other hand, Theorem~\ref{thm:halfspace-example-realizable} shows that any PAC learning algorithm in the \emph{no attack} setting, would have essentially risk $1$ under \emph{targeted} poisoning.


\begin{remark}[Other loss functions] Most of our initial results in this work are proved for the 0-1 loss as the default for classification. Yet, the written proof of Theorem~\ref{thm:Learn-Realize-Proper} holds for any loss function. Theorem~\ref{thm:Learn-Realize-Improper} can also likely be extended to other “natural” losses, but using a more complicated ``decision combiner'' than the majority. In particular, the learner can now output a label for which ``most'' sub-models will have ``small" risk (parameters most/small shall be chosen carefully). The existence of such a label can probably be proved by a similar argument to the written proof of the 0-1 loss. However, this operation is not poly time.

\end{remark}

\subsection{Distribution-specific learning}
Our previous results are for distribution-independent learning. This  still leaves open to study distribution-specific learning. That is, when the input distribution is fixed, one might able to prove stronger results. 



We then study the learnability of halfspaces under instance-targeted poisoning on \textit{the uniform distribution over the unit sphere}. Note that one can map all the examples in the $d$-dimensional space to the surface of the unit sphere, and their relative position to a homogeneous halfspace remains the same. Hence, one can limit both $\omega$ and instance $x \in  \R^d \sm {0^d}$ to be unit vectors in $\sphere{d-1}$. Therefore, distributions $D_\cX$ on the unit sphere surface can represent any distribution in the $d$-dimensional space. For example, a $d$-dimensional isotropic Gaussian distribution can be equivalently mapped to the uniform distribution over the unit sphere as far as classification with homogeneous halfspaces is concerned. We note that when the attack is \emph{non-targeted}, it was already shown by \cite{bshouty2002pac} that whenever $b(m)=o(m)$, then robust PAC learning is possible (if it is possible in the no-attack setting). Therefore, our results below can be seen as extending the results of \citep{bshouty2002pac} to the \emph{instance-targeted} poisoning attacks.

\begin{theorem}[Learnability of halfspaces under the uniform distribution] \label{thm:halfspace-positive}
In the realizable setting, let $D$ be uniform on the $d$ dimensional unit sphere $\sphere{d-1}$ and let adversary's budget for $\Rep_{b(m)}$ be $b(m)={cm}/{\sqrt{d}}$. Then for the halfspace hypothesis set $\cH_{\mathsf{half}}$, there exists a deterministic proper certifying learner $\RobCertLearn$ such that the following 
\begin{equation*}
    \Pr_{\cS \gets D^{m}}\left[\CCor_{\Rep_{b(m)}}(\cS, D) \geq 1 - 2\sqrt{2\pi} \cdot c - \sqrt{2\pi d} \cdot \eps \right]
\end{equation*} 
is at least $ 1 - \delta $ for sufficiently large sample complexity $m \geq \MUC(\eps, \delta)$, where $\MUC$ is the sample complexity of uniform convergence on $\cH_{\mathsf{half}}$. So the problem is properly and certifiably PAC learnable under $b$-replacing instance-targeted poisoning attacks.

\end{theorem}

For example, when $c= {1}/{502}$, $\eps = {c}/{(100\sqrt{d})}$ and $\delta = 0.01$, Theorem~\ref{thm:halfspace-positive} implies that
$$ \Pr_{\cS \gets D^{m}}\left[\CCor_{\Rep_{b(m)}}(\cS, D) \geq 99\% \right] \geq 99\%. $$

\begin{proof}[Proof of Theorem~\ref{thm:halfspace-positive}]
 Without loss of generality, we assume $\omega = (1, 0, 0 \dots, 0)\in \cH_{\mathsf{half}}$ denotes the ground-truth halfspace, i.e., $\mathsf{Risk}(\omega, D) = 0$. Therefore, for any data set that is i.i.d. sampled $\cS \sim D^m$, $\mathsf{Risk}(\omega, \cS) = 0$. We denote $\beta(m) = b(m)/m = c/\sqrt{d}$ be the fraction of replaced examples in the data set, and for simplicity we may use $b$ and $\beta$ to represent $b(m)$ and $\beta(m)$ in the following analysis. 

We now show that hypothesis class $\cH_{\mathsf{half}}$ is properly and certifiably PAC learnable under instance-targeted poisoning attacks on $D$. The general idea is to prove that for the majority of examples $e = (x, y) \sim D$, the risk of any hypothesis that incorrectly predicts $x$ is large. Let $\Adv_b(\cS)$ be an arbitrary adversary of budget $b(m)$. Since the adversary needs to fool the ERM algorithm, the adversary needs to change the data set from $\cS$ to $\cS'$, so that the empirical risk of a ``bad'' hypothesis $\omega'$,  $\Risk(\omega', \cS')$, is lower than the empirical risk of $\omega$, $\Risk(\omega, \cS')$. However, since the adversary can only make $b$ changes, we have 
\begin{equation*}
        \Risk(\omega, \cS')  \leq \Risk(\omega, \cS) + \beta = \beta,  \text{~~~and~~~}
       \Risk(\omega', \cS')   \geq \Risk(\omega', \cS) - \beta.
 \end{equation*}
Also, according to the uniform convergence property of the hypothesis set, let $\MUC(\eps, \delta)$ be the sample complexity of uniform convergence. Then with probability at least $1-\delta$ over $\cS$, we have $\Risk(\omega', D) \leq \Risk(\omega', \cS) + \eps$. Therefore, to fool ERM on $x$ with budget $b$, the adversary needs  \begin{equation}
    \exists \omega' \in \cH_{\mathsf{half}} \text{~~such that~~} \Risk(\omega', D) \leq 2\beta + \eps \text{~~and~~} \omega'(x) \neq \omega(x).
    \label{eq:condition_halfspace}
\end{equation}

 We then show that when $m \geq \MUC(\eps, \delta)$, for the majority of instances according to $D$, no such $\omega'$ exists if $B$ is sufficiently small.

The intersection of the halfspace $\omega$ and the $d$- dimensional sphere $\sphere{d-1}$, i.e., the ``equator'', is a $(d-1)$-dimensional sphere. Suppose $x = (x_1, x_2, \dots, x_d)$, let $\theta$ be the angle between $x$, the origin, and the halfspace $\omega$. There exists an unique $x'$ on the equator that has the minimal distance to the $x$ among all the points on the equator, and $\angle xox' = \theta$ where $o$ stands for the origin $\{0, 0, \dots, 0\}$. 
For any halfspace $\omega_1$ where $x'$ is on $\omega_1$, the angle between $\omega$ and $\omega_1$ is at least $\theta$. Therefore, a halfspace where $\omega'(x) \neq \omega(x)$ has the property that the angle between $\omega'$ and $\omega$ is at least $\theta$. In that case, since the the risk of $\omega'$ on $D$ is at least $\mathsf{Risk}(\omega', D) \geq \theta/\pi$. In the following analysis, we call an example $x'$ around angle $\theta'$ of a halfspace $\omega'$, if the angle between $x'$, the origin and halfspace $\omega'$ is less than $\theta'$.

As the distribution $D$ is uniform, the probability of an example fall into angle $\theta$ around the halfspace $\omega$ can be calculated by measuring the size of the surface within angle $\theta$, which is then upper bounded by the cylindrical surface size of a cylinder whose bottom is a $(d-1)$-dimensional unit ball and height is $2\theta$. Let $S_{d-1}$ denotes the surface of the $(d-1)$-dimensional unit sphere, then this cylinder surface has the size of $2\theta S_{d-1}$. We further denote the surface of a $d$-dimensional ball as $S_d$. Therefore, the probability of a random example falls into the set  within angle $\theta$  around $\omega$ can be upper bounded by
\begin{equation*}
    \Pr_{(x, y) \sim D}\left[x \text{ is within angle } \theta \text{ around } \omega\right] < \frac{2\theta S_{d-1}}{ S_d} < \frac{\theta\sqrt{2d}}{\sqrt{\pi}}.
\end{equation*}

The last inequality follow from Proposition~\ref{fact:sphere} in the appendix. Now, let $\theta_0 = (2\beta + \eps)\pi = {2\pi c}/{\sqrt{d}} + \pi\eps$, then ${\theta_0\sqrt{2d}}/{\sqrt{\pi}} = 2\sqrt{2\pi} \cdot c + \sqrt{2\pi d} \cdot \eps$. Therefore, we have for at least $1 - (2\sqrt{2\pi} \cdot c + \sqrt{2\pi d} \cdot \eps)$ of all possible $x$, all halfspace $\omega'$ that $\omega'(x) \neq \omega(x)$ has $\Risk(\omega', D) > 2\beta + \eps$, which according to Equation~\ref{eq:condition_halfspace}, indicates that the adversary needs budget more than $b$ to change the prediction of $x$. 
 
Finally, we define a certifying model $\hcert$ that returns certifications  $\geq b$ with high probability. For input $e=(x,y)$ and $\cS$, suppose $\omega' = \Learn(\cS)$, let $\theta'$ be the angle between $x$ and $\omega'$, then $$\hcert(x) = \begin{cases}     \max\left\{0, \left(\frac{\theta'}{2\pi} - \frac{\eps}{2}\right) \cdot m\right\} & \frac{\theta'}{\pi} \geq 2\beta + \eps  \\ 0 & \text{Otherwise}
 \end{cases}.$$ Following our analysis,  we have $\hcert(x) > b$ for all the examples that are not within angle $\theta'$ of $\omega$, which is with high probability. Also, for any $x$ that $\theta'/\pi \geq 2\beta + \eps$, we have $\forall \omega'(x) \neq \omega(x)$, $\Risk(\omega', D) \geq \theta'/\pi$. To flip the prediction on $x$, the adversary need to replace at least 
 \begin{align*}
      \beta'   \geq \frac{\min_{\omega' \in \cH}\left\{\Risk(\omega', \cS)\right\}}{2}  
       \geq \frac{\min_{\omega' \in \cH}\left\{\Risk(\omega', D) - \eps \right\}}{2} \geq \frac{{\theta}/{\pi} - \eps}{2} = \frac{\theta' }{2\pi} - \frac{\eps}{2}
 \end{align*}
 fractions of any $\cS$ that is $\eps$-representative. 
 Therefore, $\hcert$ gives a correct certification for all examples for any $\cS$ that is $\eps$-representative, and the certification result is larger than $b$ for the majority of examples for any such $\cS$.
 
 In summary, when $b ={cm}/{\sqrt{d}}$ and $m \geq \MUC(\eps, \delta)$, with probability $1 - \delta$, there are at least $1 - 2\sqrt{2\pi} \cdot c - \sqrt{2\pi d} \cdot \eps$ of examples that are robust to any $b$-replacing instance-targeted poisoning attacks. Therefore, the certifying learner $\RobCertLearn(\cS)(x) = (\Learn(\cS)(x), \hcert(x))$ gets 
 \begin{align*}
      \Pr_{\cS \gets D^{m}}\left[\CCor_{\Rep_{b(m)}}(\cS, D) \geq 1 - 2\sqrt{2\pi} \cdot c - \sqrt{2\pi d} \cdot \eps \right]   \geq 1 - \delta.
 \end{align*}
  Therefore, $\cH$ is certifiably and properly PAC learnable under $\Rep_b$ attacks.
\end{proof}

We also show that the above theorem is essentially optimal, as long as we use proper learning. Namely, for any fixed dimension $d$, with budget $b=O(m/\sqrt{d})$, a $b$-replacing adversary can guarantee success of fooling the majority of examples.
Note that for constant $d$, when $m \to \infty$, this is just a constant fraction of data being poisoned, yet this constant fraction can be made arbitrary small when $d\to \infty$.

\begin{theorem}[Limits of robustness of PAC learners under the uniform distribution]
\label{thm:halfspace-negative}
In the realizable setting, let $D$ be uniform over the $d$ dimensional unit sphere $\sphere{d-1}$. For the halfspace hypothesis set $\cH_{\mathsf{half}}$, if $b(m) \geq {cm}/{\sqrt{d}}$ for $b$-label flipping attacks $\Flp_b$, for any proper learner $\Learn$ one of the following two conditions holds.
Either $\Learn$ is \emph{not} a PAC learner for the hypothesis class of half spaces (even without attacks), or for sufficiently large $m \geq \MPAC({3c}/{(10\sqrt{d})}, \delta)$, with probability $1 - \sqrt{\delta + 2e^{-c^2/18}}$ over the selection of $\cS$ we have $$\Risk_{\Flp_b}( \cS, D) \geq 1 - \sqrt{\delta + 2e^{-c^2/18}},$$ where $\MPAC$ is the sample complexity of the   learner $\Learn$. 
\end{theorem}
For example, when $c= 20$ and $\delta = 0.00009$, we have
$ \Risk_{\Flp_b}( \cS, D)  \geq 99\%. $
 
\begin{proof}[Proof of Theorem~\ref{thm:halfspace-negative}]
Let $\omega \in \cH_{\mathsf{half}}$ denote the ground-truth halfspace, i.e., $\mathsf{Risk}(\omega, D) = 0$. We now design an adversary that fools the learner $\Learn$ within the budget $b(m)$. We start by proving the theorem for the ERM rule, and then we discuss how it extends to any PAC learner.

According to the concentration of the uniform measure over the unit sphere $\sphere{d-1}$ (e.g., see~\cite{matousek2013lectures}), for any set of measure 0.5 on the sphere, its $\rho$-neighborhood $T_\rho$ (defined as the set of all the points whose Euclidean distance less or equal to $\rho$) has measure
\begin{equation*}
    \mu(T_\rho) \geq 1 - 2e^{-d\rho^2/2}.
\end{equation*}

Therefore, for any halfspace $\omega$, the measure of samples that has $\rho$ distance to $\omega$ is at least $1 - 4e^{-d\rho^2/2}$.

Now, given an example $x$ and the training data set $\cS$, suppose $\theta$ is the angle between $x$ and $\omega$, the adversary $\Adv_b \in \Flp_b$ act like this:
\begin{enumerate}
    \item Rotate $\omega$ to $x$ by $\theta$. Let $\omega'$ denotes the result halfspace (where $x$ landed on).
    \item Rotate $\omega'$ with another $\theta$ in the same direction to the halfspace $\omega''$.
    \item For any example from the data set $\cS$ that is between $\omega$ and $\omega''$, flip its label.
    \item Return the data set as $\cS'$.
\end{enumerate}

    
    Let $\rho_0 = {c}/{3\sqrt{d}}$, then at least $1 - 2e^{-c^2/18}$ of $x$ has at most $\rho_0$ distance to $\omega$. The probability measure of the surface between $\omega$ and $\omega''$ is $2\theta/\pi$, where $2\theta/\pi \leq 2\sin(\theta) \leq 2\rho_0$. 
Let $\MUC(\eps, \delta)$ be the sample complexity of uniform convergence. Then with probability at least $1 - \delta$ over $\cS$, we have $\Risk(\omega'', \cS) \leq \Risk(\omega'', \cD) + \eps \leq 2\rho_0 + \eps$. 
 
 
 Let $\eps = 0.9\rho_0$, then the adversary flips $\Risk(\omega'', \cS) \cdot m$ examples, which with probability $1 - \delta$ we have $\Risk(\omega'', \cS) \leq 2.9\rho_0 < b/m$. Now, the ERM learner will go for the hypothesis with the minimal error on $\cS'$,  which is then $\omega''$. As $\omega''(x) \neq \omega(x)$, the ERM learner will give a wrong answer on $x$. With probability $1 - \delta$, the adversary will complete the attack within budget $b$ on at least $1 - 2e^{-c^2/18}$ examples, by the union bound, the adversary succeeds on $1 - \delta - 2e^{-c^2/18}$ examples. Finally, by an averaging argument, we have with probability $1 - \sqrt{\delta + 2e^{-c^2/18}}$, the adversary succeeds with $ 1 -  \sqrt{\delta + 2e^{-c^2/18}}$ examples.


\paragraph{Extension to any proper PAC learner}

To extend the result to any proper PAC learner, we use a similar proof as in Theorem~\ref{thm:halfspace-example-realizable}. We show same $\Adv_b$ can be extended to fool any proper PAC learner with high probability.

Let $D'$ be the ``poisoned'' distribution, that for $\cS' = \Adv_b(\cS)$, we have $\cS' \sim (D')^m$. Then with probability $1 - \delta$, we have $\Risk(\Learn(\cS'), \cD') \geq \Risk(\Learn(\cS'), \cS') - \eps$.
Now, let $ \MPAC(\eps_1, \delta_1)$ be the sample complexity of $\Learn$ on $D'$. When $m \geq \MPAC(\eps_1, \delta_1)$, on the distribution $D'$, $\Learn(\cS')$ holds $\Risk(\Learn(\cS'), D') \leq \eps_1$ with probability at least $1 - \delta_1$. 

Let $\eps_1 = 0.9\rho_0 = {3c}/{10\sqrt{d}}$. Because hypothesis set $\cH$ are halfspaces, the prediction region (the subset of all the examples predicted for a specific label) is connected. Therefore, if $\Learn(\cS')$ incorrectly predicts $x$ (which is, correctly predicts $x$ on the original data set $\cS$),  as $x$ is on $\omega'$, at least half of the surface between $\omega$ and $\omega''$ is incorrectly predicted, i.e., $\Risk(\Learn(\cS'), D') \geq \rho_0$. This contradicts $\Risk(\Learn(\cS'), D') \leq \eps_1 = 0.9\rho_0$. Therefore, with probability $1 - \delta_1$, the adversary will complete the attack within budget $b$ on at least $1 - 2e^{-c^2/18}$ examples, by the union bound, the adversary succeeds on $1 - \delta_1 - 2e^{-c^2/18}$ examples. Finally, by an averaging argument, we have with probability $1 - \sqrt{\delta_1 + 2e^{-c^2/18}}$, the adversary succeeds with $ 1 -  \sqrt{\delta_1 + 2e^{-c^2/18}}$ examples.



\end{proof}



\subsection{Relating  risk and robustness}

Risk uses a worst-case budget to capture what an adversary can do, while robustness does so using an average-case budget.  
Theorem~\ref{thm:Risk-Rob} below relates the two notions of risk and robustness in the context of targeted poisoning attacks and   is  inspired by results previously proved for adversarial inputs that are crafted during test-time attacks (\cite{diochnos2018adversarial,mahloujifar2019curse}). In particular, Theorem~\ref{thm:Risk-Rob} proves that for 0-1 loss, it is equivalent to fully understand either of them to understand the other one and allows to derive numerical values for one through the other. 


\begin{theorem}[From risk to robustness and back] \label{thm:Risk-Rob}
Suppose $\cS \in (\cX \times \cY)^m$ is a training set,   $\Learn$ is a   learner,    $D$ is a distribution over $\cX \times \cY$,  $\AdvC_b$ is an adversary class with the budget $b$, and $\AdvC=\cup_{b \in \N} \AdvC_b$. Then the following relations hold. 
\begin{enumerate}
    \item  {\bf From robustness to risk.} For any non-negative loss function, we have
    $$\Risk_{\AdvC_b}(\cS,r,D) = \int_0^\infty \Pr_{e \sim D}\left[\Rob^\tau_\AdvC(\cS,r,e) \leq b\right] \cdot \mathrm{d}\tau.$$ 
    For the special case of 0-1 loss, this simplifies to $\Risk_{\AdvC_b}(\cS,r,D) = \Pr_{e \sim D}\left[\Rob_\AdvC(\cS,r,e) \leq b\right]$.  
    \item {\bf From risk to robustness.} Suppose we use the 0-1 loss. Suppose  $b$ is large enough such that $\Risk_{\AdvC_n}(\cS,r,D)=1$, or equivalently $\Cor_{\AdvC_i}(\cS,r,D)=0$ for $i\geq b$.\footnote{For example, if the adversarial strategy allows flipping up to $b$ labels, then for $b = m$ the adversary can flip all the labels. For natural hypothesis classes and learning algorithms, changing all the labels allows the adversary to control prediction on all points and so  $\Risk_{\AdvC_b}(\cS,D)=1$.} Then,   it holds that
    \begin{align*}
 \Rob_\AdvC(\cS,r,D) &= b -\sum_{i=0}^{b-1} \Risk_{\AdvC_i}(\cS,r,D) 
 \\ &= \sum_{i=0}^{b-1} \Cor_{\AdvC_i}(\cS,r,D) 
  \\&= \sum_{i=0}^{\infty} \Cor_{\AdvC_i}(\cS,r,D).        
    \end{align*}

In other words, if we could compute adversarial risks for all $b$, we can also compute the average robustness by summing robust correctness. 
\end{enumerate}
\end{theorem}

\begin{proof}[Proof of Theorem~\ref{thm:Risk-Rob}]
We write the proof for deterministic learners who do not have any randomness, but the same exact proof works when a randomness $r$ exists and is fixed. 

By Definition~\ref{def:TarRob}, for any threshold $\tau$ we have
\begin{align*}
     \Rob^\tau_\AdvC(\cS, e) \leq b \iff  \sup_{\cS' \in \AdvC_b(\cS)} \left\{ \loss(\Learn(\cS'), e)\right\} \geq \tau  \\
     \iff \exists  \cS' \in \AdvC_b(\cS), \loss(\Learn(\cS')(x), y) \geq \tau.
\end{align*}
Also, the so-called expectation through CDF\footnote{See \url{https://en.wikipedia.org/w/index.php?title=Expected_value&oldid=1017448479\#Basic_properties} as accessed on May 16, 2021.} implies that for a non-negative function $f$ and a distribution $D$, we have
\begin{equation}
\label{eq:expectation_CDF}
    \Ex_{x \sim D}\left[f(x)\right] = \int_{\tau=0}^{\infty} \Pr\left[f(x) \geq \tau\right] \mathrm{d}\tau
\end{equation}
Therefore, Part 1 can be proven as follows.
\begin{equation*}
\begin{split}
 \Risk_{\AdvC_b}(\cS,D) & = \Ex_{e \sim D} \left[\loss_{\AdvC_b}(\cS, e)\right] \\ \text{(by Definition~\ref{def:TarRisk})} &= \Ex_{e \sim D}\left[\sup_{\cS' \in \AdvC_b(\cS)} \left\{\loss(\Learn(\cS'), e)\right\} \right] 
 \\
 \text{(by Equation~\ref{eq:expectation_CDF})} & = \int_{\tau = 0}^{\infty} \Pr_{e \sim D}\left[\sup_{\cS' \in \AdvC_b(\cS)} \left\{\loss(\Learn(\cS'), e)\right\} \geq \tau \right] \cdot \mathrm{d}\tau  
 \\ \text{(by Definition~\ref{def:TarRob})} &= \int_{\tau = 0}^{\infty} \Pr_{e \sim D}[\Rob^\tau_{\AdvC}(\cS, e) \leq b]  \cdot \mathrm{d}\tau  .    
\end{split}
\end{equation*}

We now prove Part 2. From Definition~\ref{def:TarRob}, $\Rob_\AdvC(\cS, e) \in \N \cup \{\infty\}$. We then have 
\begin{equation}
    \forall i\in \R, \Pr\left[\Rob_\AdvC(\cS, e) \geq i\right] = \Pr\left[\Rob_\AdvC(\cS, e) \geq \lceil i \rceil\right], 
    \label{eq:piecewise_rob}
\end{equation}
where $\lceil i \rceil$ is the ceiling function that returns the minimum integer above $i$. Furthermore, recall that $b$ is a large enough number that for any example $e$, $\forall i \geq b, \Risk_{\AdvC_i}(\cS, e) = 1$ and $\Cor_{\AdvC_i}(\cS, e) = 0$. We have  $\forall e, \Pr\left[\Rob_{\AdvC}(\cS, e) \leq b \right] = 1$, i.e., $\Rob_{\AdvC}(\cS, e) \leq b$. Then we conclude that,

\begin{equation*}
\begin{split}
\Rob_{\AdvC}(\cS, D) &= \Ex_{e \sim D}\left[\Rob_{\AdvC}(\cS, e)\right] \\ 
\text{(by Equation~\ref{eq:expectation_CDF})} &= \int_{\tau=0}^{\infty} \Pr_{e \sim D} \left[\Rob_{\AdvC}(\cS, e) \geq \tau \right] \cdot  \mathrm{d} \tau \\
\text{(by Equation~\ref{eq:piecewise_rob})} &= \sum_{i=0}^{\infty} \Pr_{e \sim D} \left[\Rob_{\AdvC}(\cS, e) > i\right]   \\
&= b - \sum_{i=0}^{b-1} \Pr_{e \sim D}\left[\Rob_{\AdvC}(\cS, e) \leq i\right] \\
\text{(by Definition~\ref{def:TarRob})}   &= b - \sum_{i=0}^{b-1}  \Risk_{\AdvC_i}(\cS, D) \\
\text{(by Definition 
\ref{def:TarRisk})} &= \sum_{i=0}^{b-1} \Cor_{\AdvC_i}( \cS, D)  = \sum_{i=0}^{\infty} \Cor_{\AdvC_i}( \cS, D). \qedhere
\end{split}
\end{equation*}
\end{proof}



 
\section{Experiments}

\begin{figure*}[ht]
    \centering
    \includegraphics[width=0.4\textwidth]{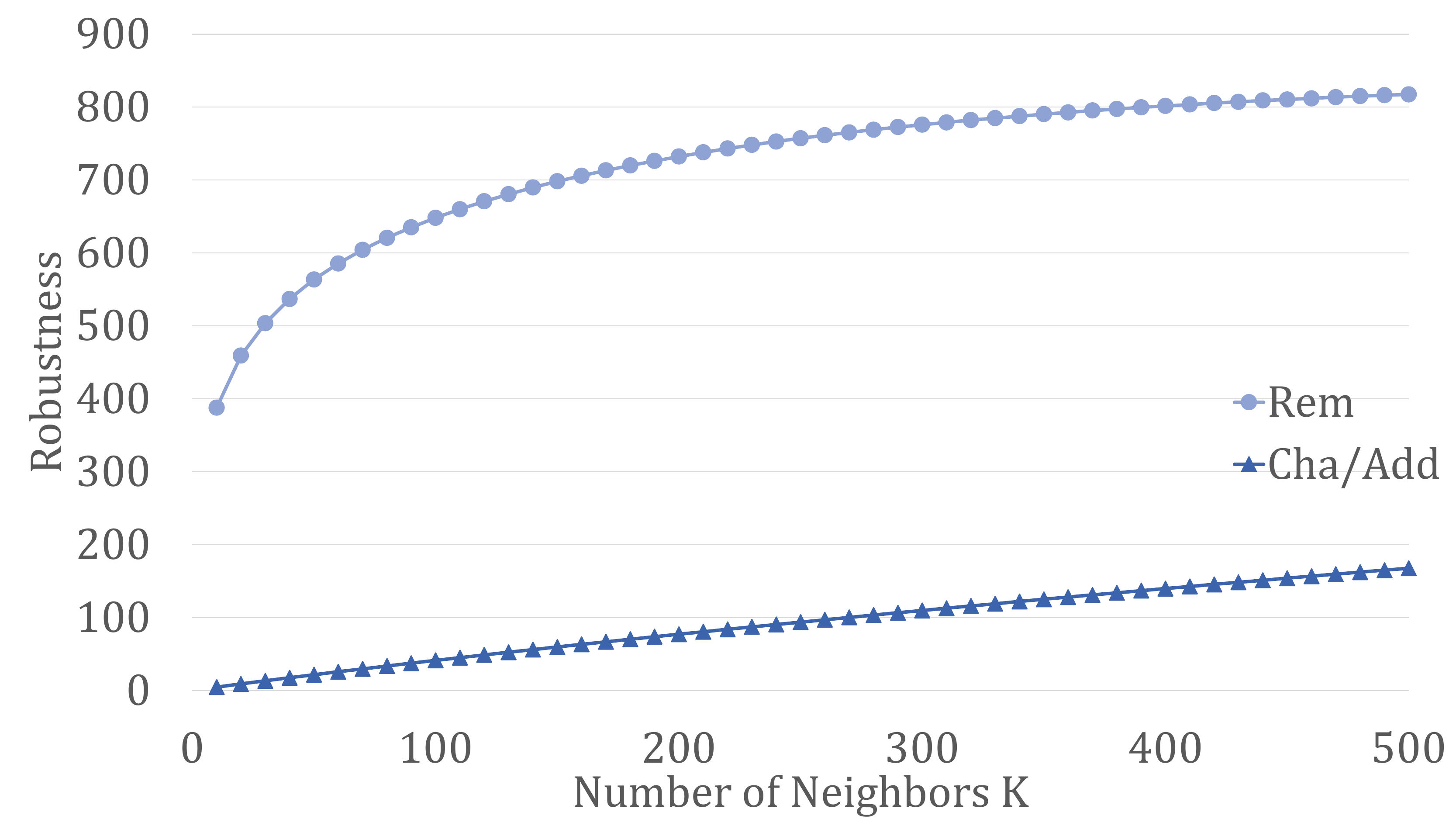} \hspace{1cm}
    \includegraphics[width=0.4\textwidth]{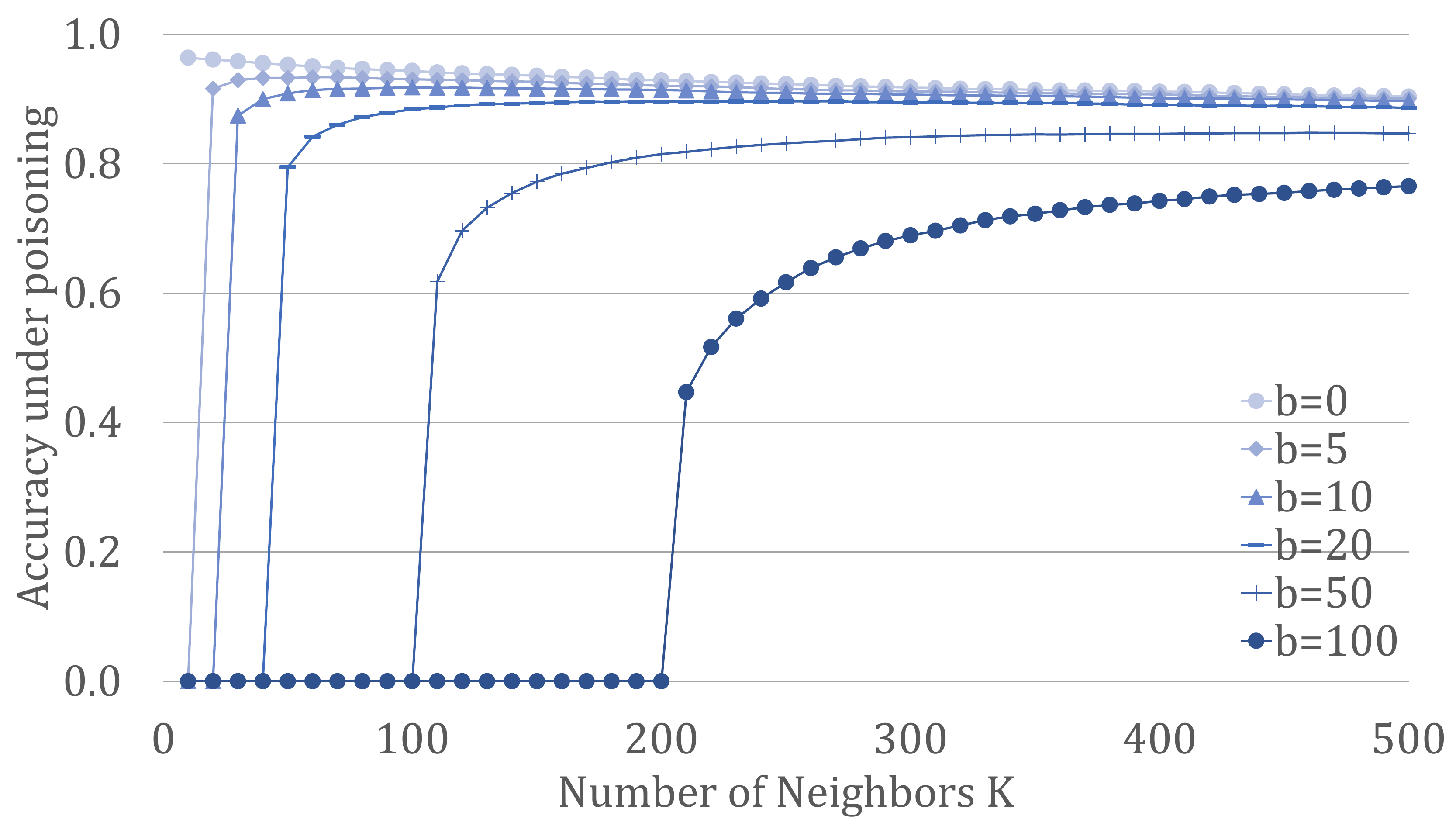} \\
    (a) \hspace{8cm} (b)
    \caption{Experiment of $K$-Nearest Neighbors on the MNIST dataset. (a) The trend of Robustness $\Rob(\Learn_{\mathsf{knn}},\cS_{\mathsf{MNIST}}, \cD)$ on attacks $\Rep$, $\Add$, and $\Rem$, with the increase of number of neighbors $K$. (b) Accuracy of $K$-NN model under $\Rep_b$ with different poisoning budget $b$.}
    \label{fig:experiment_knn}
\end{figure*}

\begin{figure*}[ht]
\centering
\includegraphics[width=0.38\textwidth]{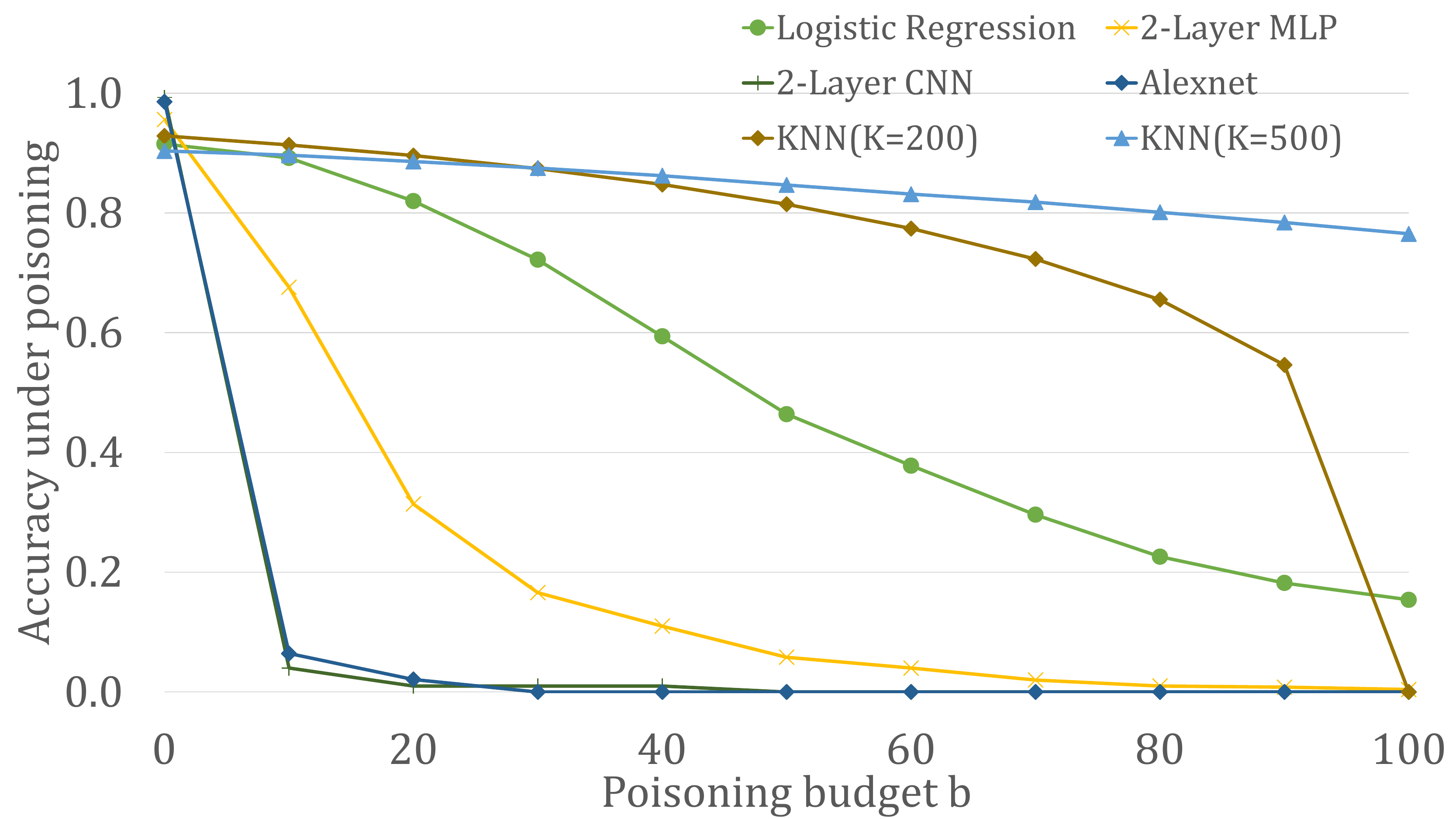} \hspace{1cm}
\includegraphics[width=0.38\textwidth]{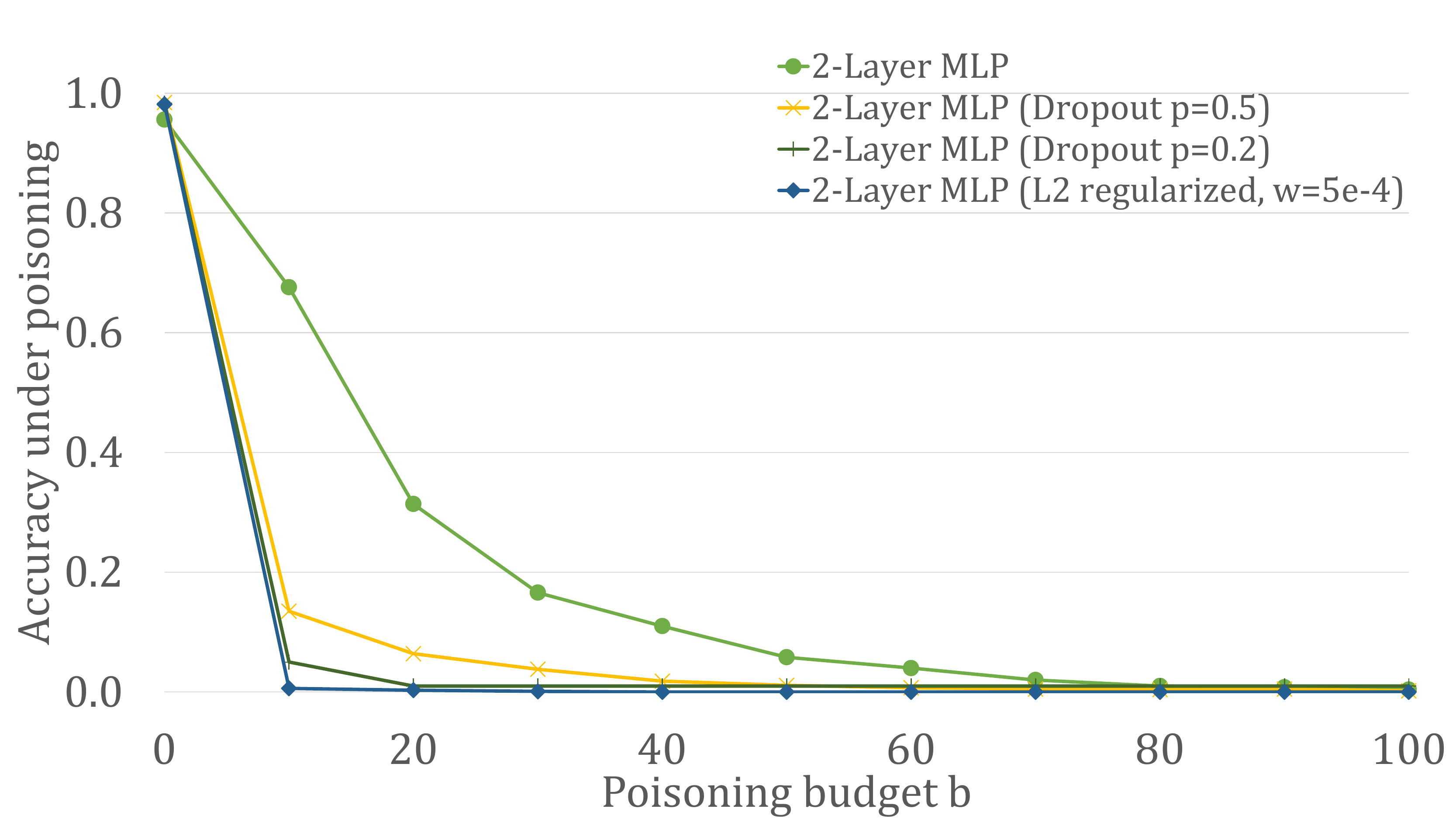} \\
(a) \hspace{8cm} (b)
\caption{Accuracy of different learners under $\Add_b$ instance-targeted poisoning on the MNIST dataset. (a) Compare different learners. (b) Compare dropout and regularization mechanics on Neural Networks.}
\label{fig:experiment_all_mnist}
\end{figure*}
In this section, we study the power of instance-targeted poisoning on the MNIST dataset \citep{lecun1998gradient}. We first analyze the robustness of  $K$-Nearest Neighbor model, where the robustness can be efficiently calculated empirically. We then empirically study the accuracy under targeted poisoning for  multiple other different learners. Previous empirical analysis on instance-targeted poisoning (e.g., \cite{shafahi2018poison}) mostly focus on clean-label attacks. In this work, we use attacks of any labels, which lead to stronger attacks compared to clean-label attacks. We also study multiple models in our experiment, while previous work mostly focus on neural networks, and we then compare the performance of different models under the same attack.

$K$-Nearest Neighbor ($K$-NN) is non-parameterized model that memorizes every training example in the dataset. This special structure of $K$-NN allows us to empirically evaluate the robustness to poisoning attacks. The $K$-NN model in this section uses the majority vote defined below.
\begin{defi}[$K$-NN learner]  \label{def:knn}
    For training dataset $\cS$ and example $e = (x, y)$, let $\cN(x)$ denote the set of $K$ closest examples from $\cS$  $e$. Then the prediction of the $K$-NN is 
$$
       \modelknn(x) = \argmax_{j \in \cY} \sum_{(x_i, y_i) \in \cN(x)} \one [{y_i = j}].    
$$
   \end{defi}
From our definition of poisoning attack and robustness, we can measure the robustness empirically by the following lemma. Similar ideas can also be found in \citep{jia2020intrinsic}. 
\begin{lemma}[Instance-targeted Poisoning Robustness of the $K$-NN learner]
Let $\mathsf{margin}(\modelknn,  e) $ be defined as $0$ if $\modelknn(x)\neq y$ and be defined as 
$$ 
\sum_{(x_i, y_i) \in \cN(x)} \one[y_i = y] - \max_{j \in \cY, j\neq y} \sum_{(x_i, y_i) \in \cN(x)} \one[y_i = j]$$
otherwise. We then have 
$$\Rob_{\Rep_b}(\Learn_{\mathsf{KNN}}, \cS, e) = \left\lceil{\frac{\mathsf{margin}(\Learn_{\mathsf{KNN}}(\cS),  e)}{2}}\right\rceil.$$
\label{lemma:knn}
\end{lemma}

\begin{proof}[Proof of Lemma~\ref{lemma:knn}]
Following Definition~\ref{def:knn}, the prediction for a sample $x$ totally depends on the neighbor set $\cN(x)$. By definition, $\cN(x)$ is a subset of $\cS$. For the adversary class $\Rep_b$ (which can be extend to any adversary with budget $b$), they can only make at most $b$ changes to the set $\cS$, which includes at most $b$ changes to $\cN(x)$. 

For an example $e = (x, y)$, to flip the prediction to $y'$, we need to change $\cN(x)$ to $\cN'(x)$ such that $\sum_{(x_i, y_i) \in \cN'(x)} \one[{y_i = y'}] \geq \sum_{(x_i, y_i) \in \cN'(x)} \one[{y_i = y}]$. However, we have $\forall y' \neq y$, 
\begin{align*}
 \sum_{(x_i, y_i) \in \cN(x)} \one[y_i = y] - \sum_{(x_i, y_i) \in \cN(x)} \one[y_i = y'] \\ \geq \mathsf{margin}(\modelknn, e).   
\end{align*}

At least $\left\lceil{\frac{\mathsf{margin}(\Learn_{\mathsf{KNN}}(\cS),  e)}{2}}\right\rceil$ replacements needs to be made in this case. To make it work, the adversary can replace the label of  $\left\lceil{\frac{\mathsf{margin}(\Learn_{\mathsf{KNN}}(\cS),  e)}{2}}\right\rceil$ examples of label $y$ in $N(x)$ with $y'$. Therefore, we have $\Rob_{\Rep_b}(\Learn_{\mathsf{KNN}}, \cS, e) = \left\lceil{\frac{\mathsf{margin}(\Learn_{\mathsf{KNN}}(\cS),  e)}{2}}\right\rceil$.
\end{proof}

Using Lemma~\ref{lemma:knn}, one can compute the robustness of the $K$-NN model empirically by calculating the margin for every $e$ in the distribution. We then use the popular digit classification dataset MNIST to measure the robustness. 

In the experiment, we use the whole training dataset to train (60, 000 examples), and evaluate the robustness on the testing dataset (10, 000 examples). We calculate the robustness under $\Rep_b$, $\Rem_b$, and $\Add_b$ attacks. We measure the result with different number of neighbors $K$ present the result in Figure~\ref{fig:experiment_knn}a. We also measure the accuracy under poisoning of $\Rep_b$ and report it in Figure~\ref{fig:experiment_knn}b.
The results in Figure~\ref{fig:experiment_knn} indicates the following message. (1) From Figure~\ref{fig:experiment_knn}a, when the number of neighbors $K$ increases, the robustness also increases as expected. The robustness of $K$-NN to $\Rep$ and $\Add$ increases almost linearly with $K$. (2) The robustness to $\Rem$ is much larger than to $\Rep$ and $\Add$. $\Rem$ is a more difficult attack in this scenario. (3) From Figure~\ref{fig:experiment_knn}b, when the number of neighbors $K$ increases, the models' accuracy without poisoning slightly decreases. (4) From Figure~\ref{fig:experiment_knn}b, $K$-NN keeps around 80\% accuracy to $b=100$ instance-targeted poisoning when $K$ becomes large.

For general learners, measuring their robustness provably under attacks is harder because there is no clear efficient attack that is provably optimal. In this case, we perform a heuristic attack to study the power of $\Add_b$. The general idea is that for an example $e = (x, y)$, we poison the dataset by adding $b$ copies of $(x, y')$ into the dataset with the second best label $y'$ in $h(x)$, where $b$ is the Adversary's budget. We then report the accuracy under poisoning with different budget $b$ on classifiers including Logistic regression, 2-layer Multi-layer Perceptron (MLP), 2-layer Convolutional Neural Network (CNN), AlexNet  and also $K$-NN in Figure~\ref{fig:experiment_all_mnist}a. We get the following conclusion: (1) Models that have low risk without poisoning, such as MLP, CNN and AlexNet, typically have low empirical error, which makes it less robust under  poisoning. (2) $K$-NN with large $K$ have high accuracy under poisoning compared to other models by   sacrificing its clean-label prediction accuracy. 

Finally, in Figure~\ref{fig:experiment_all_mnist}b we report on our findins about two regularization mechanics, dropout and $L2$-regularization, on the Neural Network learner and whether adding them can provide better robustness against instance-targeted poisoning $\Add_b$. We use a 2-layer Multi-layer Perceptron (MLP) as the base learner and adds dropout/regularization to the learner. From the figure, we get the following messages: (1) Dropout and regularization help to improve the accuracy without the attacks (when $b=0$). (2) These mechanics don't help the accuracy with the $\Add_b$ attacks. The accuracy under attack is  worse than the vanilla Neural Network. We   conclude that these simple mechanics cannot help the neural net to defend against instance-targeted poisoning.

\remove{
\begin{acknowledgements}
 Mohammad Mahmoody and Ji Gao were supported by NSF  CCF-1910681 and CNS-1936799. Amin Karbasi was supported by NSF IIS-1845032 and ONR N00014-19-1-2406. 
\end{acknowledgements}
}  

\bibliography{Biblio/abbrev0, Biblio/crypto, Biblio/mlcrypto, Biblio/OtherRefs}

\begin{thebibliography}{37}
\providecommand{\natexlab}[1]{#1}
\providecommand{\url}[1]{\texttt{#1}}
\expandafter\ifx\csname urlstyle\endcsname\relax
  \providecommand{\doi}[1]{doi: #1}\else
  \providecommand{\doi}{doi: \begingroup \urlstyle{rm}\Url}\fi

\bibitem[Barreno et~al.(2006)Barreno, Nelson, Sears, Joseph, and
  Tygar]{barreno2006can}
Marco Barreno, Blaine Nelson, Russell Sears, Anthony~D Joseph, and J~Doug
  Tygar.
\newblock Can machine learning be secure?
\newblock In \emph{Proceedings of the 2006 ACM Symposium on Information,
  computer and communications security}, pages 16--25. ACM, 2006.

\bibitem[Blum et~al.(2021)Blum, Hanneke, Qian, and Shao]{blum2021robust}
Avrim Blum, Steve Hanneke, Jian Qian, and Han Shao.
\newblock obust learning under clean-label attack.
\newblock In \emph{Conference on Learning Theory}, 2021.

\bibitem[Bshouty et~al.(2002)Bshouty, Eiron, and Kushilevitz]{bshouty2002pac}
Nader~H Bshouty, Nadav Eiron, and Eyal Kushilevitz.
\newblock Pac learning with nasty noise.
\newblock \emph{Theoretical Computer Science}, 288\penalty0 (2):\penalty0
  255--275, 2002.

\bibitem[Chen et~al.(2018)Chen, Carvalho, Baracaldo, Ludwig, Edwards, Lee,
  Molloy, and Srivastava]{chen2018detecting}
Bryant Chen, Wilka Carvalho, Nathalie Baracaldo, Heiko Ludwig, Benjamin
  Edwards, Taesung Lee, Ian Molloy, and Biplav Srivastava.
\newblock Detecting backdoor attacks on deep neural networks by activation
  clustering.
\newblock \emph{arXiv preprint arXiv:1811.03728}, 2018.

\bibitem[Chen et~al.(2020)Chen, Li, Wu, Sheng, and Li]{chen2020framework}
Ruoxin Chen, Jie Li, Chentao Wu, Bin Sheng, and Ping Li.
\newblock A framework of randomized selection based certified defenses against
  data poisoning attacks, 2020.

\bibitem[Diakonikolas and Kane(2019)]{diakonikolas2019recent}
Ilias Diakonikolas and Daniel~M. Kane.
\newblock Recent advances in algorithmic high-dimensional robust statistics,
  2019.

\bibitem[Diakonikolas et~al.(2016)Diakonikolas, Kamath, Kane, Li, Moitra, and
  Stewart]{diakonikolas2016robust}
Ilias Diakonikolas, Gautam Kamath, Daniel~M Kane, Jerry Li, Ankur Moitra, and
  Alistair Stewart.
\newblock Robust estimators in high dimensions without the computational
  intractability.
\newblock In \emph{Foundations of Computer Science (FOCS), 2016 IEEE 57th
  Annual Symposium on}, pages 655--664. IEEE, 2016.

\bibitem[Diochnos et~al.(2018)Diochnos, Mahloujifar, and
  Mahmoody]{diochnos2018adversarial}
Dimitrios~I Diochnos, Saeed Mahloujifar, and Mohammad Mahmoody.
\newblock Adversarial risk and robustness: general definitions and implications
  for the uniform distribution.
\newblock In \emph{Proceedings of the 32nd International Conference on Neural
  Information Processing Systems}, pages 10380--10389, 2018.

\bibitem[Diochnos et~al.(2019)Diochnos, Mahloujifar, and
  Mahmoody]{diochnos2019lower}
Dimitrios~I Diochnos, Saeed Mahloujifar, and Mohammad Mahmoody.
\newblock Lower bounds for adversarially robust pac learning.
\newblock \emph{arXiv preprint arXiv:1906.05815}, 2019.

\bibitem[Etesami et~al.(2020)Etesami, Mahloujifar, and
  Mahmoody]{etesami2020computational}
Omid Etesami, Saeed Mahloujifar, and Mohammad Mahmoody.
\newblock Computational concentration of measure: Optimal bounds, reductions,
  and more.
\newblock In \emph{Proceedings of the Fourteenth Annual ACM-SIAM Symposium on
  Discrete Algorithms}, pages 345--363. SIAM, 2020.

\bibitem[Gao et~al.(2021)Gao, Karbasi, and Mahmoody]{gao2021learning}
Ji~Gao, Amin Karbasi, and Mohammad Mahmoody.
\newblock Learning and certification under instance-targeted poisoning, 2021.

\bibitem[Goldblum et~al.(2020)Goldblum, Tsipras, Xie, Chen, Schwarzschild,
  Song, Madry, Li, and Goldstein]{goldblum2020data}
Micah Goldblum, Dimitris Tsipras, Chulin Xie, Xinyun Chen, Avi Schwarzschild,
  Dawn Song, Aleksander Madry, Bo~Li, and Tom Goldstein.
\newblock Data security for machine learning: Data poisoning, backdoor attacks,
  and defenses.
\newblock \emph{arXiv preprint arXiv:2012.10544}, 2020.

\bibitem[Gu et~al.(2017)Gu, Dolan-Gavitt, and Garg]{gu2017badnets}
Tianyu Gu, Brendan Dolan-Gavitt, and Siddharth Garg.
\newblock Badnets: Identifying vulnerabilities in the machine learning model
  supply chain.
\newblock \emph{arXiv preprint arXiv:1708.06733}, 2017.

\bibitem[Ji et~al.(2017)Ji, Zhang, and Wang]{ji2017backdoor}
Yujie Ji, Xinyang Zhang, and Ting Wang.
\newblock Backdoor attacks against learning systems.
\newblock In \emph{2017 IEEE Conference on Communications and Network Security
  (CNS)}, pages 1--9. IEEE, 2017.

\bibitem[Jia et~al.(2020)Jia, Cao, and Gong]{jia2020intrinsic}
Jinyuan Jia, Xiaoyu Cao, and Neil~Zhenqiang Gong.
\newblock Intrinsic certified robustness of bagging against data poisoning
  attacks.
\newblock \emph{arXiv preprint arXiv:2008.04495}, 2020.

\bibitem[Kearns and Li(1993)]{kearns1993learning}
Michael Kearns and Ming Li.
\newblock Learning in the presence of malicious errors.
\newblock \emph{SIAM Journal on Computing}, 22\penalty0 (4):\penalty0 807--837,
  1993.

\bibitem[Koh and Liang(2017)]{pmlr-v70-koh17a}
Pang~Wei Koh and Percy Liang.
\newblock Understanding black-box predictions via influence functions.
\newblock In Doina Precup and Yee~Whye Teh, editors, \emph{Proceedings of the
  34th International Conference on Machine Learning}, volume~70 of
  \emph{Proceedings of Machine Learning Research}, pages 1885--1894. PMLR,
  06--11 Aug 2017.
\newblock URL \url{http://proceedings.mlr.press/v70/koh17a.html}.

\bibitem[Lai et~al.(2016)Lai, Rao, and Vempala]{lai2016agnostic}
Kevin~A Lai, Anup~B Rao, and Santosh Vempala.
\newblock Agnostic estimation of mean and covariance.
\newblock In \emph{Foundations of Computer Science (FOCS), 2016 IEEE 57th
  Annual Symposium on}, pages 665--674. IEEE, 2016.

\bibitem[LeCun et~al.(1998)LeCun, Bottou, Bengio, and
  Haffner]{lecun1998gradient}
Yann LeCun, L{\'e}on Bottou, Yoshua Bengio, and Patrick Haffner.
\newblock Gradient-based learning applied to document recognition.
\newblock \emph{Proceedings of the IEEE}, 86\penalty0 (11):\penalty0
  2278--2324, 1998.

\bibitem[Levine and Feizi(2021)]{levine2020deep}
Alexander Levine and Soheil Feizi.
\newblock Deep partition aggregation: Provable defenses against general
  poisoning attacks.
\newblock In \emph{International Conference on Learning Representations}, 2021.
\newblock URL \url{https://openreview.net/forum?id=YUGG2tFuPM}.

\bibitem[Mahloujifar and Mahmoody(2017)]{mahloujifar2017blockwise}
Saeed Mahloujifar and Mohammad Mahmoody.
\newblock Blockwise p-tampering attacks on cryptographic primitives,
  extractors, and learners.
\newblock In \emph{Theory of Cryptography Conference}, pages 245--279.
  Springer, 2017.

\bibitem[Mahloujifar and Mahmoody(2019)]{mahloujifar2019can}
Saeed Mahloujifar and Mohammad Mahmoody.
\newblock Can adversarially robust learning leveragecomputational hardness?
\newblock In \emph{Algorithmic Learning Theory}, pages 581--609. PMLR, 2019.

\bibitem[Mahloujifar et~al.(2018)Mahloujifar, Diochnos, and
  Mahmoody]{mahloujifar2018learning}
Saeed Mahloujifar, Dimitrios~I Diochnos, and Mohammad Mahmoody.
\newblock Learning under $ p $-tampering attacks.
\newblock In \emph{Algorithmic Learning Theory}, pages 572--596. PMLR, 2018.

\bibitem[Mahloujifar et~al.(2019{\natexlab{a}})Mahloujifar, Diochnos, and
  Mahmoody]{mahloujifar2019curse}
Saeed Mahloujifar, Dimitrios~I Diochnos, and Mohammad Mahmoody.
\newblock The curse of concentration in robust learning: Evasion and poisoning
  attacks from concentration of measure.
\newblock In \emph{Proceedings of the AAAI Conference on Artificial
  Intelligence}, volume~33, pages 4536--4543, 2019{\natexlab{a}}.

\bibitem[Mahloujifar et~al.(2019{\natexlab{b}})Mahloujifar, Mahmoody, and
  Mohammed]{mahloujifar2019universal}
Saeed Mahloujifar, Mohammad Mahmoody, and Ameer Mohammed.
\newblock Universal multi-party poisoning attacks.
\newblock In \emph{International Conference on Machine Learing (ICML)},
  2019{\natexlab{b}}.

\bibitem[Matousek(2013)]{matousek2013lectures}
Jiri Matousek.
\newblock \emph{Lectures on discrete geometry}, volume 212.
\newblock Springer Science \& Business Media, 2013.

\bibitem[Montasser et~al.(2019)Montasser, Hanneke, and Srebro]{montasser2019vc}
Omar Montasser, Steve Hanneke, and Nathan Srebro.
\newblock Vc classes are adversarially robustly learnable, but only improperly.
\newblock In \emph{Conference on Learning Theory}, pages 2512--2530. PMLR,
  2019.

\bibitem[Papernot et~al.(2016)Papernot, McDaniel, Sinha, and
  Wellman]{papernot2016towards}
Nicolas Papernot, Patrick McDaniel, Arunesh Sinha, and Michael Wellman.
\newblock Towards the science of security and privacy in machine learning.
\newblock \emph{arXiv preprint arXiv:1611.03814}, 2016.

\bibitem[Rosenfeld et~al.(2020)Rosenfeld, Winston, Ravikumar, and
  Kolter]{rosenfeld2020certified}
Elan Rosenfeld, Ezra Winston, Pradeep Ravikumar, and Zico Kolter.
\newblock Certified robustness to label-flipping attacks via randomized
  smoothing.
\newblock In \emph{International Conference on Machine Learning}, pages
  8230--8241. PMLR, 2020.

\bibitem[Shafahi et~al.(2018)Shafahi, Huang, Najibi, Suciu, Studer, Dumitras,
  and Goldstein]{shafahi2018poison}
Ali Shafahi, W~Ronny Huang, Mahyar Najibi, Octavian Suciu, Christoph Studer,
  Tudor Dumitras, and Tom Goldstein.
\newblock Poison frogs! targeted clean-label poisoning attacks on neural
  networks.
\newblock \emph{arXiv preprint arXiv:1804.00792}, 2018.

\bibitem[Sloan(1995)]{Sloan::Noise:four-types}
Robert~H. Sloan.
\newblock {Four Types of Noise in Data for {PAC} Learning}.
\newblock \emph{Information Processing Letters}, 54\penalty0 (3):\penalty0
  157--162, 1995.

\bibitem[Steinhardt et~al.(2017)Steinhardt, Koh, and
  Liang]{steinhardt2017certified}
Jacob Steinhardt, Pang~Wei Koh, and Percy Liang.
\newblock Certified defenses for data poisoning attacks.
\newblock In \emph{Proceedings of the 31st International Conference on Neural
  Information Processing Systems}, pages 3520--3532, 2017.

\bibitem[Turner et~al.(2019)Turner, Tsipras, and Madry]{turner2019label}
Alexander Turner, Dimitris Tsipras, and Aleksander Madry.
\newblock Label-consistent backdoor attacks.
\newblock \emph{arXiv preprint arXiv:1912.02771}, 2019.

\bibitem[Valiant(1984)]{valiant1984theory}
Leslie~G Valiant.
\newblock A theory of the learnable.
\newblock \emph{Communications of the ACM}, 27\penalty0 (11):\penalty0
  1134--1142, 1984.

\bibitem[Valiant(1985)]{valiant1985learning}
Leslie~G Valiant.
\newblock Learning disjunction of conjunctions.
\newblock In \emph{IJCAI}, pages 560--566, 1985.

\bibitem[Wang et~al.(2019)Wang, Yao, Shan, Li, Viswanath, Zheng, and
  Zhao]{wang2019neural}
Bolun Wang, Yuanshun Yao, Shawn Shan, Huiying Li, Bimal Viswanath, Haitao
  Zheng, and Ben~Y Zhao.
\newblock Neural cleanse: Identifying and mitigating backdoor attacks in neural
  networks.
\newblock In \emph{2019 IEEE Symposium on Security and Privacy (SP)}, pages
  707--723. IEEE, 2019.

\bibitem[Weber et~al.(2020)Weber, Xu, Karlas, Zhang, and Li]{weber2020rab}
Maurice Weber, Xiaojun Xu, Bojan Karlas, Ce~Zhang, and Bo~Li.
\newblock Rab: Provable robustness against backdoor attacks.
\newblock \emph{arXiv preprint arXiv:2003.08904}, 2020.

\end{thebibliography}
\appendix

\section{Useful facts}

\begin{fact}\label{fact:middle-binom}
The function $\binom{2k}{k} \sqrt{k}/4^k$ is increasing for $k \in \N$, $\binom{2k}{k} \sqrt{k+1}/4^k$ is decreasing for $k \in N$, and the limit of both when $k \to \infty $ is $1/\sqrt{\pi}$. Therefore, the following holds for all positive $k$,  
$$ \frac{4^k}{\sqrt{(k+1)\cdot \pi}} \leq \binom{2k}{k}  \leq \frac{4^k}{\sqrt{k\pi}}.$$
\end{fact}

\begin{fact} \label{fact:sphere}
Let $S_{d-1}$ be the area of the surface of the unit ball in $d$ dimensions. Then the following two hold.

\begin{enumerate}
    \item $S_{2k} = \frac{2\cdot k! (4\pi)^k}{(2k)!} $
    \item $S_{2k-1}=\frac{2 \pi^k}{(k-1)!}$
\end{enumerate}
\end{fact}

The following proposition follows from Facts \ref{fact:middle-binom} and \ref{fact:sphere}.
\begin{proposition}
It holds that $$\frac{\sqrt{d-1}}{\sqrt{2\pi}} \leq \frac{S_{d-1}}{S_{d}} \leq \frac{\sqrt{d}}{\sqrt{2\pi}}.$$ 
\end{proposition}

\end{document}